\documentclass{article}

\PassOptionsToPackage{round}{natbib}
\usepackage[preprint]{neurips_2022}

\usepackage[utf8]{inputenc} % allow utf-8 input
\usepackage[T1]{fontenc}    % use 8-bit T1 fonts
\usepackage[hidelinks]{hyperref}       % hyperlinks
\usepackage{url}            % simple URL typesetting
\usepackage{booktabs}       % professional-quality tables
\usepackage{amsfonts}       % blackboard math symbols
\usepackage{nicefrac}       % compact symbols for 1/2, etc.
\usepackage{microtype}      % microtypography
\usepackage{xcolor}         % colors
\usepackage{dsfont}
\usepackage{bm}
\usepackage{mathtools}
\usepackage{cleveref}
\usepackage{amsmath}
\usepackage{amssymb}
\usepackage{amsthm}
\usepackage{xifthen}
\usepackage[ruled]{algorithm2e}
\usepackage{algorithmic}

\newcommand{\lr}[1]{\left(#1\right)}
\newcommand{\lrs}[1]{\left[#1\right]}

\newcommand{\lrc}[1]{\left\{#1\right\}}

\newcommand{\Dmin}{\Delta_{\text{min}}}
\newcommand{\argmin}{\text{argmin}}
\newcommand{\UCB}{\text{UCB}}
\newcommand{\LCB}{\text{LCB}}
\newcommand{\Nin}{N^{\mathrm{in}}}
\newcommand{\Nout}{N^{\mathrm{out}}}

\newcommand\numberthis{\addtocounter{equation}{1}\tag{\theequation}}
\newcommand{\prob}[1][]{\ifthenelse{\isempty{#1}{}}{\mathbb P}{\mathbb P \left[#1\right]}}
\newcommand{\E}[1][]{\ifthenelse{\isempty{#1}{}}{\mathbb E}{\mathbb E \left[#1\right]}}
\newcommand{\1}[1][]{\ifthenelse{\isempty{#1}{}}{\mathds 1}{\mathds 1 \left[#1\right]}}

\newcommand{\tmin}{t_{\mathrm{min}}}
\newcommand{\mas}{\mathrm{mas}}

\DeclareRobustCommand{\VAN}[3]{#2} % proper Dutch 'van/de' capitalisation

\newtheorem{prop}{Proposition}
\newtheorem*{prop*}{Proposition}

\newtheorem{theorem}{Theorem}
\newtheorem{lemma}{Lemma}
\newtheorem*{lemma*}{Lemma}

\newtheorem{definition}{Definition}

\title{A Near-Optimal Best-of-Both-Worlds Algorithm\\
for Online Learning with Feedback Graphs}
\author{
 Chloé Rouyer \\
 Dept. of Computer Science\\
 University of Copenhagen, Denmark \\
  \texttt{chloe@di.ku.dk} \\
   \And
   Dirk van der Hoeven \\
   Dept. of Computer Science \\
   Università degli Studi di Milano, Italy \\
   \texttt{dirk@dirkvanderhoeven.com} \\
   \AND
  Nicolò Cesa-Bianchi  \\
  DSRC \& Dept. of Computer Science \\
   Università degli Studi di Milano, Italy \\
  \texttt{nicolo.cesa-bianchi@unimi.it}
   \And
    Yevgeny Seldin\\
  Dept. of Computer Science\\
University of Copenhagen, Denmark \\
   \texttt{seldin@di.ku.dk} \\
}

\begin{document}

\maketitle

\begin{abstract}
We consider online learning with feedback graphs, a sequential decision-making framework where the learner's feedback is determined by a directed graph over the action set. We present a computationally efficient algorithm for learning in this framework that simultaneously achieves near-optimal regret bounds in both stochastic and adversarial environments. The bound against oblivious adversaries is $\tilde{O} (\sqrt{\alpha T})$,
where $T$ is the time horizon and $\alpha$ is the independence number of the feedback graph. The bound against stochastic environments is $O\big(\lr{ \ln T}^2 \max_{S\in \mathcal I\lr{G}} \sum_{i \in S} \Delta_i^{-1}\big)$ 
where $\mathcal I\lr{G}$ is the family of all independent sets in
a suitably defined undirected version of the graph
and $\Delta_i$ are the suboptimality gaps.
The algorithm combines ideas from the EXP3++ algorithm for stochastic and adversarial bandits and the EXP3.G algorithm for feedback graphs with a novel exploration scheme. The scheme, which exploits the structure of the graph to reduce exploration, is key to obtain best-of-both-worlds guarantees with feedback graphs. 
We also extend our algorithm and results to a setting where the feedback graphs are allowed to change over time.
\end{abstract}

\section{Introduction}
\label{sec:intro}

Online learning is a general framework for studying sequential decision-making in unknown environments (see, for example, \citep{CBL06,BCB12, orabona2019modern}). 
We consider a setting where, at each round, the player chooses an action (a.k.a.\ arm) from a fixed set of $K$ actions and incurs the loss associated with the chosen action. The performance of the learner is quantified in terms of regret, which is the difference between the total loss incurred by the learner over the duration of the game, and the smallest cumulative loss obtained by a player that would only ever play the same action throughout the game.

The smallest achievable regret is determined by a number of parameters.
% There are several variants in the formulation of this problem, which determine its difficulty.
One of these parameters is the amount of feedback that the learner receives at each round. There is a whole spectrum of problems, characterized by the amount of feedback received by the learner. At the one extreme of this spectrum is the bandit setting, where the learner only observes the loss of the action taken. % This allows to study the trade-off between exploration and exploitation.
At the other extreme is the full information setting, where the learner observes the full loss vector at the end of each round, irrespective of the action played. 

There are two common ways to interpolate between full information and bandit feedback. One is to allow the learner to make a limited number of additional observations without restricting how the additional observations are selected. Then no additional observations correspond to the bandit setting and $K-1$ additional observations correspond to the full information setting. This way of interpolation was proposed by \citet{SBCA14} in two variants, "prediction with limited advice" and "multiarmed bandits with paid observations". It was also studied by \citet{TS18}.

The second way of interpolation, which we focus on in this paper, is via feedback graphs \citep{ACBGMMS14}. In this setting observations of the learner are governed by a feedback graph on the actions. When an action is played, the learner observes the losses of all of its neighbors in the feedback graph. A complete graph corresponds to the full information setting, whereas a graph containing only self-loops corresponds to the bandit setting. This setting has multiple variants, depending on whether the graph is directed or undirected, observed or unobserved, static or dynamic.

Another important parameter characterizing online learning problems is the type of environment. The two primary types that we focus on are stochastic and adversarial environments. In stochastic environments each action is associated with a fixed, but unknown distribution, and in each round the loss of each action is sampled independently from the corresponding distribution.
In adversarial environments the loss sequence is chosen arbitrarily. We consider oblivious adversarial environments, where the loss sequences are chosen independently of the actions taken by the learner. 

For a long time stochastic and adversarial environments where studied separately, but in practice the exact nature of environment is rarely known. In recent years this has led to a growing interest in ``best-of-both-worlds'' algorithms that are robust against adversarial loss sequences and, at the same time, provide tighter regret guarantees in the stochastic regime. Most work has focused on the bandit setting \citep{BS12,SS14,AC16,SL17,WL18}, where the Tsallis-INF algorithm proposed by \citet{ZS19,ZS21} was shown to achieve the optimal regret rates in both stochastic and adversarial regimes, as well as a number of intermediate regimes. The analysis was further improved by \citet{MS21} and \citet{ito2021parameter}. In the full information setting \citet{MG19} have shown that the well-known Hedge algorithm originally designed for the adversarial setting \citep{LW94} also achieves the optimal stochastic regret. Best-of-both-world results also spilled over to other domains, including additional approaches to full information games and online convex optimization \citep{koolen2016combining,VanErven2021metagrad,negrea2021minimax}, decoupled exploration and exploitation \citep{RS20}, combinatorial bandits \citep{zimmert2019beating}, bandits with switching costs \citep{rouyer2021algorithm}, MDPs \citep{JL20,JHL21}, and linear bandits \citep{LLW+21}.

In the context of online learning with feedback graphs the only best-of-both-worlds result known to us is by \citet{EK21} for undirected graphs. They present an intricate FTRL-based algorithm with a regularization function that is a product of the  Tsallis and Shannon entropies.
% They focus on a restriction of the problem where the feedback graph is partitioned into cliques.
% \citeauthor{EK21}
The algorithm simultaneously enjoys an $O\big(\sqrt{\chi T} \big(\ln(KT)\big)^2\big)$ pseudo-regret bound in the adversarial regime and an $O\big(\big(\ln(KT)\big)^4 \sum_{k}\frac{\ln T}{\Delta_k})$ pseudo-regret bound in the stochastic regime, where $T$ is the number of prediction rounds, $\chi$ is the clique covering number of the undirected feedback graph, and the summation in the second bound is on the smallest non-zero gap within each clique.

It is tempting to apply an FTRL-based algorithm with Tsallis entropy regularization to online learning with feedback graphs, since Tsallis entropy with power $a = 1/2$ leads to the optimal Tsallis-INF algorithm for the bandit setting \citep{ZS21} and Tsallis entropy with power $a=1$ leads to the Hedge algorithm, which is optimal in the full information setting. However, as also noted by \citet{EK21}, extension of the analysis to online learning with feedback graphs when the power $a \in (1/2, 1)$ is not straightforward and, so far, there was no success in this direction. %, and \citet{EK21} require a much more complex regularization function to derive their results.
Furthermore, at the moment it is unclear whether it is possible to derive bounds that take further advantage of the graph structure and depend on the independence number of the graph when $a < 1$.

\textbf{Our contribution.}
We significantly extend and improve on the bounds of \citet{EK21}. Our results hold for directed graphs (with self-loops), depend on the independence number of the graph, have a better dependence on $T$ in the stochastic regime, and extend to time-varying feedback graphs.
Our approach takes advantage of the common structure shared by two exponential weights algorithms: EXP3.G \citep{ACBDK15} and EXP3++ \citep{SS14, SL17}, to obtain near-optimal best-of-both worlds guarantees. By using similar ideas as in the proof of the regret bound of EXP3.G, the proposed algorithm adapts to the independence number of the graph. We derive an $O(\sqrt{\tilde \alpha T \ln K})$ pseudo-regret bound against adversarial sequences of losses, where $\tilde \alpha$ is the strong independence number, which is a graph dependent quantity smaller than the clique covering number. For undirected graphs, independence number and strong independence number are equal
and the result matches 
% the best known upper bound \citep{ACBDK15} up to constants, and 
the best known lower bound $\Omega(\sqrt{\alpha T})$ within logarithmic factors \citep{ACBGMMS14}.
In the stochastic setting we use the idea of injected exploration from EXP3++ to estimate the suboptimality gaps of each arm. By introducing a novel dynamic exploration set and an appropriate exploration rate, we derive an almost optimal regret bound in the stochastic setting. Along the way, we also improve the regret bound of EXP3++ in the stochastic bandit setting. Our exploration set is constructed by sorting the arms by ascending gap estimates, and then adding a new arm to the exploration set if the arm cannot be observed by playing another arm previously added to the set. If we play each arm $i$ in the exploration set at a rate $1/\hat \Delta_{i}^2$, where $\hat \Delta_i$ is the gap estimate, then all arms $j$ in the graph are observed with probability at least $1/\hat \Delta_{j}^2$. 
% This allows to control the impact of the exploration on the regret without reducing the rate at which observations are collected.
% To present our main result we will first introduce some notation. Let $G = (V, E)$ be a directed feedback graph with independence number $\alpha$ (where the independence number is computed on $G$ ignoring edge directions) and let $\mathcal{I}(G)$ be the collection of all independence sets in the bidirectional subgraph of $G$; i.e., the subgraph $G' = (V',E')$ such that $V = V'$ and $(i,j) \in E'$ if and only if $(i,j) \in E$ and $(j,i) \in E$ (note that $\alpha \le \max\big\{|S| \,:\, S\in\mathcal{I}(G)\big\}$). Using these definitions we can present an informal statement of our main result.

To present our main result we introduce some notations. Let $G = (V, E)$ be a directed feedback graph with independence number $\alpha$ (where the independence number is computed on $G$ ignoring edge directions). We define a strongly independent set on $G$ as an independent set on the subgraph $G' = (V,E')$, where $(i,j) \in E'$ if and only if $(i,j) \in E$ and $(j,i) \in E$. We use $\tilde \alpha$ to denote the strong independence number of $G$,  and  $\mathcal{I}(G)$ to denote the collection of all the strongly independent sets in $G$.
We note that $\alpha = \tilde \alpha$ for undirected graphs and $\alpha \le \tilde \alpha$ for directed graphs.
% Furthermore, we recall that cliques are fully connected subgraphs, connected by bi-directional edges, meaning that $\alpha \leq \tilde \alpha \leq \chi$. Thus, regret bound scaling with the independence number is superior to scaling with the clique number
Now we can present an informal statement of our main result.
\begin{theorem}[Informal] Given a directed feedback graph $G = (V, E)$ with strong independence number $\tilde \alpha$,
there exists an algorithm (Algorithm \ref{alg:exp3G++}) whose pseudo-regret can simultaneously be bounded by
% $O (\sqrt{\alpha T} \ln\lr{KT})$ 
$O (\sqrt{\tilde \alpha T \ln K})$ 
against adversarial loss sequences and by
$O\big(\lr{\ln T}^2 \max_{S\in \mathcal I\lr{G}} \sum_{i \in S} \Delta_i^{-1}\big)$
% (K^2/\alpha)  \sum_{i \in V} \Delta_i^{-3} \lr{\ln \lr{(K^2/\alpha) \Delta_i^{-4}}}^2$
against stochastic loss sequences.
\end{theorem}
We emphasize that Algorithm \ref{alg:exp3G++} requires neither prior knowledge of the type of the environment (adversarial or stochastic), nor the time horizon.
% In the stochastic regime, we observe that the term $O((K^2/\alpha)  \sum_{i \in V} \Delta_i^{-3} \lr{\ln \lr{(K^2/\alpha) \Delta_i^{-4}}}^2)$ is independent of time, meaning that for large values of $T$ we recover a $O(\lr{\ln T}^2 \max_{S\in \mathcal I\lr{G}} \sum_{i \in S} \Delta_i^{-1})$ bound. 

\subsection{Additional Related Work}

The study of bandits with feedback graphs was initiated by \citet{MS11} in the adversarial regime and by \citet{CKLB12} in the stochastic regime. % In those works the feedback graphs are fixed in time, undirected, and contain all self-loops. Their regret bounds scale with the clique number $\theta$ of the graph. Subsequent work have relaxed assumptions on the graph structure, and aim to derive dependencies on the independence number of the graph. 
In the adversarial regime, the optimal regret rates for arbitrary directed graphs were characterized (up to log factors) by \citet{ACBDK15}. They showed an $\Omega(T)$ lower bound for graphs that have non-observable nodes (i.e., with an empty in-neighborhood). For graphs with observable nodes, they derived pseudo-regret bounds of order $O\big(\sqrt{\alpha T} \log(KT)\big)$ when all nodes are strongly observable (i.e., they have a self-loop or their in-neighborhood contains all of the other nodes) and of order $O\big((\delta \ln K)^{1/3}T^{2/3}\big)$ for weakly observable graphs (where each non-strongly observable node is in the out-neighborhood of some observable node). Here $\alpha$ is the independence number of the graph and $\delta$ is the dominating number of the weakly observable portion of the graph. \citet{VanderHoeven2021beyond} derived results for the multiclass classification with feedback graphs setting. 
% For graphs with self-loops, \citet{ACBGMMS14} derive a $\Omega(\sqrt{\alpha T})$ lower bound, meaning that the result of \citet{ACBDK15} is tight up to log factors.
The setting where the graph can adversarially change over time has been studied by \citet{ACBGMMS14} in the case of directed graphs with self-loops. For learners that are allowed to observe the feedback graph at the beginning of each round, they achieved a bound of $O\big(\ln K \sqrt{\ln (KT) \sum_{t = 1}^T \alpha_t}\big)$, where $\alpha_t$ is the independence number of the graph at time $t$. For the case of undirected graphs, they proved a refined bound $O\big(\sqrt{\ln K \sum_{t = 1}^T \alpha_t}\big)$ that holds even when the learner can only observe the graph at the end of each round.
Note that, as shown by \citet{CHK16}, in order to take advantage of the graph structure in the adversarial regime, it not sufficient to observe
% the feedback graph at the end each round. % \citet{CHK16} show a $\Omega(\sqrt{KT})$ lower bound when the learner can only observe
the neighborhood of the played action at the end of each round.

In the stochastic regime, \citet{BES14, BLES17} considered a fixed, possibly directed, feedback graph. They derived an asymptotic lower bound showing that the regret scales as $\Omega(c^*\ln T)$, where $c^*$---which is related to the domination number of the graph---is the solution to a linear program expressing the trade-off between the loss incurred from playing an action and the observations that can be gathered from playing that action. They proposed an algorithm that can achieve a matching $O\big(c^*\ln T + K d\big)$ pseudo-regret bound, where $d$ is the maximum degree in the feedback graph. 
In the case of graphs that change over time, \citet{CHK16} derived an $O(\sum_{i \in S} (\ln T) / \Delta_i)$ bound, where $S$ is a set containing an order of $\alpha$ arms (up to log factors), and $\alpha$ is an upper bound on the independence number of the graphs in the sequence. They achieved this result without requiring to observe the graphs fully, and having only access to the neighbourhood of the arm played at the end of the round. 
Both of these approaches are based on arm elimination algorithms, which---by construction---are not suitable for best-of-both worlds guarantees. The proof strategy of \citet{CHK16} was adapted by \citet{LTW20} to provide refined bounds for both UCB-N and Thompson Sampling-N, which are variants of UCB1 \citep{ACB+02} and Thompson Sampling \citep{Tho33}. In both cases, \citet{LTW20} considered undirected feedback graphs and obtained pseudo-regret bounds that scale as $O\big(\max_{\mathrm{Ind} \in \mathcal{I}(G)}\sum_{i \in \mathrm{Ind}} \ln (KT) (\ln T) / \Delta_i\big)$, where $\mathcal{I}(G)$ is the collection of all the independence sets of the graph.

\section{Problem Setting and Definitions}
\label{sec:pb_setting}

\paragraph{Problem Setting}
We consider a sequential decision-making game, where in each round $t=1,2,\ldots$, the learner repeatedly plays an action $I_t \in V$, where $|V| = K$, receives a feedback based on a feedback graph $G = (V, E)$, and suffers a loss $\ell_{t, I_t}$.
We consider directed feedback graphs with self-loops, meaning that $(i, i) \in E$ for each vertex $i \in V$. The feedback received by the learner at the end of round $t$ is $\big\{ (i,\ell_{t,i}) \,:\, i \in \Nout(I_t) \big\}$, where  $\Nout(i) = \{j \in V: (i, j) \in E\}$ is the out-neighbourhood of $i$. Similarly, we define $\Nin (i) = \{j \in V: (j, i) \in E\}$ to be the in-neighborhood of $i$.
For each arm $ i \in V$, $\ell_{t, i} \in [0,1]$ for $t \ge 1$. In the adversarial regime the losses are generated arbitrarily by an oblivious adversary. In the stochastic regime they are independently drawn from a fixed but unknown distribution with expectation $\mathbb{E}[\ell_{1, i}]$.
%
% The learner receives feedback $\{ (j, \ell_t(j)) : j \in \Nout(I_t)\}$.
%
The performance of the learner is measured in terms of the pseudo-regret:
\[
\mathcal{R}_T =  \E[\sum_{t = 1}^T \ell_{t, I_t}] -  \min_{i \in V}\E[ \sum_{t = 1}^T\ell_{t, i}].
\]
In the stochastic regime, we define the best arm $i^*$ as the arm with the smallest expected loss, i.e. $i^*=\argmin_{i \in V} \mathbb{E}[\ell_{1, i}]$. The pseudo-regret can then be expressed in terms of the suboptimality gaps $\Delta_i = \mathbb{E}[\ell_{1, i}- \ell_{1, i^*}]$,
\begin{equation}
    \mathcal{R}_T = \sum_{t = 1}^T \sum_{i\in V} \E[p_{t, i}]\Delta_i,
\end{equation}
where $p_{t, i}$ is the probability that the learner plays action $i$ at round $t$.
%
% \paragraph{Notation}
We define the smallest suboptimality gap $\Dmin = \min_{i : \Delta_i > 0} \lrc{\Delta_i}$, and for all $i$, we define $\bar\Delta_i = \max \lrc{\Dmin, \Delta_i}$, so that $\bar \Delta_{i^*}=\Dmin$.
We use $\E_t$ to express expectation conditioned on all randomness up to round $t$.

\paragraph{Properties of Graphs}

% The feedback observed by the learner depends on a feedback graph so we present here some properties of the feedback graph.

% When the learner plays arm $i$, she receives feedback for all arms $j \in \Nout(i)$, and if she wants to receive feedback for arm $i$, she needs to play an arm $j \in \Nin(i)$.
% In the case of undirected graphs, the in-neighborhood and the out-neighborhood of each arm are equal.

% The complexity of a bandits with feedback graph problem depends on how complex the feedback graph is. In particular,
% Bandits with feedback graphs allow to interpolate between the full information (experts) feedback setting (when the graph is fully connected, including self-loops) and the bandit setting (when the graph only contains self-loops and no other edges).
% If the graph is fully connected, meaning that for all $i, j \in V$, $(i, j) \in E$, then at each round the learner can observe the entirety of the loss vector independently of the arm she played and we have a full-information problem. On the opposite, if the feedback graph only contains self-loops and no other edges then we have a bandit feedback.

% We recall the following well-known graph-theoretic notions.

 Recall that a dominating set in $G$ is a subset $D \subseteq V$, such that for all $i \in V$ there exists $j \in D$, such that $(j, i) \in E$. % In other words, the learner can gather information on all arms while only playing actions that are in a dominating set.
 An independent set in $G$ is a subset $S \subseteq V$, such that for all $i, j \in S$,
 % are connected by an edge in
 $(i, j) \not\in E$ and $(j, i) \not\in E$. We define the independence number $\alpha(G)$ as the size of the largest independent set in the graph $G$. For clarity, we restate below here the definition of the strong independence number which was already mentioned in the introduction.
 %
% When working with a non directed graph $G$, if you have an independence set $S$ such that you cannot add any other vertex in $S$ while maintaining the independence of the set, then $S$ is a dominating set of $G$. This is not necessarily the case for directed graphs. In certain instances, it is useful to reduce a directed graph to a non directed one. 
%
% To characterize the complexity of problems using directed feedback graphs in the stochastic regime, we
% % reduce the feedback graph to a bidirectional subgraph graph. % that does not contain edges for
% introduce the notion of bidirectional subgraph of a feedback graph.
% %
% \begin{definition}
% \label{def:bidirectional_subgraph}
% Let $G = (V, E)$ be a directed graph. We define the bidirectional subgraph of $G$ as the undirected graph $G' = (V, E')$ such that for all $i, j \in V$, $(i, j) \in E'$ if and only if $(i, j) \in E$ and $(j, i) \in E$.
% \end{definition}
% % By construction, an undirected graph is equal to its bidirectional subgraph.
% We use $\mathcal I(G)$ to denote the collection of all independent sets in the bidirectional subgraph of $G$.
\begin{definition}
\label{def:strong_indep}
Let $G = (V, E)$ be a directed graph. We define a strongly independent set on $G$ as an independent set on the subgraph $G' = (V,E')$, where $(i,j) \in E'$ if and only if $(i,j) \in E$ and $(j,i) \in E$. 
Furthermore, we define $\tilde \alpha(G)$ as the independence number of the subgraph $G'$. 
\end{definition}
We use $\mathcal{I}(G)$ to denote a collection of all the strongly independent sets in $G$. We note that $\alpha = \tilde \alpha$ for undirected graphs and $\alpha \le \tilde \alpha$ for directed graphs. 
% For any graph, we have $\alpha \leq \tilde \alpha \leq \chi$, where $\chi$ is the clique dominating number of the graph.

% Undirected graphs are a special case of directed graphs. Consider an undirected graph $G = (V, E)$. The associated directed graph is a graph $G' = (V, E')$ where for all $i, j \in V$ such that  $(i, j) \in E$, then both $(i, j) \in E'$ and $(j, i) \in E'$. Note that the bidirectional subset of $G'$ is $G$.

\section{Algorithm}
\label{sec:algo}

We present the EXP3.G++ algorithm (Algorithm \ref{alg:exp3G++}), which is a combination of the EXP3.G algorithm of  \citet{ACBDK15} and the EXP3++ algorithm of \citet{SL17}
% with a novel strategy to tune the extra exploration. 
with % a crucial new addition:
a novel exploration scheme described in Algorithm \ref{alg:Explore_set}. This scheme ensures that the additional feedback the learner obtains (relative to the bandit setting) is used nearly optimally. 

To understand the motivation behind the novel exploration scheme,
% consider the following.
note that in the stochastic setting
% due to technical reasons explained below,
EXP3.G++ needs to ensure that the loss of each arm is observed sufficiently often. However, if we would play each arm too often, the regret would scale with the number of arms, rather than with the independence number or some other graph-theoretic quantity. To avoid that, we exploit the central property of feedback graphs: since we can gather information on certain arms
% without actually playing them
by playing adjacent arms in the graph,
% we circumvent the need to play all arms. Instead,
we can restrict exploration to a subset of nodes and yet obtain sufficient information on \emph{all} the arms. 
%
% In the adversarial setting, we require upper and lower bounds on the exploration of all arms. In the stochastic regime, we require upper bounds on the rate at which we play all arms, but only lower bounds on the rate at which we observe each arm.
%
%Thanks to the feedback graph, it is possible to gather observations on the losses of all arms without requiring to play them all. 
%
We exploit this observation, to design a strategy for selecting an exploration set $S_t$ at each round $t$.
%, which ensures that we observe the losses of \emph{all} arms sufficiently often to guarantee good regret bounds in both the stochastic and adversarial regime. %, and then use different exploration rates for arms that are inside and outside $S_t$.
$S_t$ is defined in terms of estimated suboptimality gaps $\hat \Delta_{t, i}$, which are maintained by EXP3.G++. Crucially, the exploration set ensures that, with high probability, the empirical gaps are reliable estimates of the true suboptimality gaps $\Delta_i$. In turn, this ensures that we observe the loss of each arm sufficiently often.

\begin{algorithm}[tb]
	\caption{EXP3.G++}
	\label{alg:exp3G++}
\begin{algorithmic}
	\STATE {\bfseries Input: }{Feedback graph $G = (V, E)$,\\
	Learning rates $\eta_1 \geq \eta_2 \geq \dots > 0$; exploration rates $\varepsilon_{t, i}$ for $i \in V$, see \Cref{eq:def_eti}
	}\\
	\STATE {\bfseries Initialize: }{$\tilde L_0 = \bm0_K$, $\hat L_0 = \bm0_K$ and $O_0 = \bm0_K$. Play each arm once to initialize $\hat L$ and $O$}\\
	\FOR{$t = K+1, K+2, \dots$}
		\STATE ${\displaystyle \forall i \in V: \text{UCB}_{t, i} = \min\lrc{1, \frac{\hat L_{t-1, i}}{O_{t-1, i}} + \sqrt{\frac{\gamma \ln\lr{tK^{1/\gamma}}}{2O_{t-1, i}}}} }$ \\
		\STATE ${\displaystyle \forall i \in V: \text{LCB}_{t, i} = \max\lrc{0, \frac{\hat L_{t-1, i}}{O_{t-1, i}} - \sqrt{\frac{\gamma \ln\lr{tK^{1/\gamma}}}{2O_{t-1, i}}}} }$ \\
		\STATE ${\forall i \in V: \hat \Delta_{t, i} = \max\lrc{0, \text{LCB}_{t, i} - \min_j \text{UCB}_{t, j}}}$\\
		\STATE {$\forall i\in V:$ update $\varepsilon_{i,t}$ based on the gap estimates $\hat \Delta_{t}$} \\
\STATE $\displaystyle \forall i\in V: q_{t, i} =\frac{\exp(-\eta_t \tilde L_{t, i})}{\sum_{i \in V}\exp(-\eta_t \tilde L_{t, j})},\;	p_{t, i} = \left(1 - \sum_{j \in V}\varepsilon_{t, j} \right) q_{t, i} + \varepsilon_{t, i}$\\
		\STATE Sample $I_t \sim p_t$ and play it \\
		\STATE Observe $\{ (j, \ell_{t,j}) : j \in \Nout(I_t)\}$ and suffer $\ell_{t, I_t}$.  \\
		\STATE ${\displaystyle \forall i \in V: \tilde \ell_{t, i} = \frac{\ell_{t, i} \1[i \in \Nout(I_t)]}{P_{t, i}} }$, where $P_{t, i} = \sum_{j \in \Nin(i)} p_{t, j}$\\
		\STATE$\forall\  i \in V: \quad \tilde L_{t, i} = \tilde L_{t-1, i} + \tilde \ell_{t, i}$\\
		\STATE ${\displaystyle \forall i \in V: \hat  L_{t, i} = \hat L_{t-1, i}+ \ell_{t, i}}\1[i \in \Nout(I_t)]$ and ${O_{t, i}= O_{t-1, i} +  \1[i \in \Nout(I_t)]}$
		\ENDFOR
\end{algorithmic}
\end{algorithm}

%EXP3.G++ uses Algorithm \ref{alg:Explore_set} as a subroutine to construct an exploration set $S_t$ at round $t$ using the gap estimates $\hat \Delta_{t-1}$. 

The construction of the exploration set $S_t$ is detailed in Algorithm \ref{alg:Explore_set}, which is used by EXP3.G++ to update the exploration rates $\varepsilon_{t, i}$ according to \Cref{eq:def_eti}. Algorithm \ref{alg:Explore_set} starts by sorting the arms according to their gap estimates in ascending order. The exploration set is then greedily constructed by sequentially selecting the next arm with the smallest $\hat \Delta_{t, i}$, and discarding all the arms in the  out-neighborhood of that arm. The exploration set can be constructed in $O(K^3)$ time, but note that we only need to recompute it only when the order of the estimated suboptimality gaps changes. The exploration set $S_t$ has several useful properties, as shown in Proposition~\ref{prop:exploration_set} below.

\begin{algorithm}[tb]
	\caption{Exploration Set Construction}
	\label{alg:Explore_set}
\begin{algorithmic}
	\STATE {\bfseries Input: }{$K$ arms with associated gaps: $\Delta_1,\Delta_2, \dots$
	}\\
	\STATE {\bfseries Initialize: }{Exploration set $S = \emptyset$.\\
	Let $\Lambda$ be the list of arms sorted in ascending order of their associated gaps.}\\
	\FOR{$i \in \Lambda$}
	    \STATE {Add $i$ to $S$}
	    \FOR{$j \in \Nout(i)$}
	        \STATE {remove $j$ from $\Lambda$}
		\ENDFOR
	\ENDFOR
	\STATE {\bfseries Output: }{$S$}
\end{algorithmic}
\end{algorithm}

\begin{prop}
\label{prop:exploration_set}
Let $G = (V, E)$ be a directed feedback graph on $K$ arms with self-loops, and let $\hat \Delta_1, \dots, \hat \Delta_K$ be a sequence of suboptimality gaps estimates.
Let $S$ be the exploration set constructed by Algorithm \ref{alg:Explore_set} based on the sequence of suboptimality gaps. 
Then $S$ is a dominating set of $G$ with the following property: for all $i \in V$ there exists $j \in S$, such that $i \in \Nout(j)$ and $ \hat \Delta_{j} \leq \hat \Delta_{i}$. 
Furthermore, $S$ is also a strongly independent set of $G$.
\end{prop}
\begin{proof}
Let $S$ be the output of Algorithm~\ref{alg:Explore_set}.
Since $G$ contains self-loops, if $i \in S$, then $i \in \Nout(i)$ and $\hat \Delta_{i} \leq \hat \Delta_{i}$. 
If $i \not\in S$, then $i$ was removed from $\Lambda$ because $i \in \Nout(j)$ for some $j$ that, in a previous iteration, was added to $S$. Since $j$ was considered before $i$, we must have $\hat \Delta_{j} \leq \hat \Delta_{i}$.
Now, for all $i, j \in S$
% , such that $\hat \Delta_i \leq \hat \Delta_j$
, we know by construction that $j \not\in \Nout(i)$. Thus $(i, j)$ is not a directed edge in $G$, and so $S$ is a strongly independent set in $G$.
\end{proof}
We define the exploration rates at round $t$ in terms of the exploration set $S_t$, which is constructed using the aforementioned procedure. For all arms $i$ in $V$, 
\begin{equation}
    \varepsilon_{t, i} = \min \lrc{\frac{1}{2K}, \frac{1}{2} \sqrt{\frac{\lambda \ln K}{tK^2}}, \xi_{t, i}}, \label{eq:def_eti}
\end{equation}
for some constant $\lambda \in [1, K]$ and where  $\xi_{t, i}$ depends on whether $i \in S_t$ or not:
\begin{equation}
    \xi_{t, i} = \begin{cases}
    ({\beta \ln t})/({t \hat\Delta_{t, i}^2}), &\qquad \text{if}\, i \in S_t,\\
    {4}/{t^2}, &\qquad \text{otherwise},
\end{cases} \label{eq:def_xiti}
\end{equation}
where $\beta > 0$ is a constant. The role of $\xi_{t,i}$ changes depending on whether we are in an adversarial or stochastic environment. In an adversarial environment, we use $4/t^2 \leq \xi_{t,i}$ to ensure that we sample each arm with a small positive probability, which is essential to bound the second-order term in the regret bound in terms of the independence number. Note that $\varepsilon_{t, i} \leq \frac{1}{2} \sqrt{\frac{\lambda \ln K}{tK^2}}$, so choosing $\lambda = \tilde \alpha$ ensures that the cost of exploration is bounded by $\tilde{O}(\sqrt{\tilde \alpha T})$.
% in the adversarial setting. 
In the stochastic environment, the construction of the exploration set and the choice of $\xi_{t, i}$ ensure that, at each round $t$, each $i \in V$ is observed with probability at least $({\beta \ln t})/({t \hat\Delta_{t, i}^2})$, independently of whether $i$ is in the exploration set at round $t$.
% This allows us to control our loss and gap estimates. 

Formally, our procedure ensures that we can lower bound the probability with which any arm is observed. In the algorithm we use $P_{t, i} = \prob[i \in \Nout(I_t)]$ to denote the probability that arm $i$ is observed at round $t$. We can lower bound this quantity by only considering the minimum rate at which each arm is observed according to the exploration rate $\varepsilon_{t, i}$ and our construction of the exploration sets.
We use $o_{t, i}$ to denote that quantity, and we have for all $t$ and $i$,
\begin{equation}
   P_{t, i} \geq o_{t, i} = \min \lrc{\frac{1}{2K}, \frac{1}{2} \sqrt{\frac{\lambda \ln K}{tK^2}},  \frac{\beta \ln t}{t \hat\Delta_{t, i}^2}}. \label{eq:oti_lb}
\end{equation}
The definition of $o_{t, i}$ uses that $S_t$ is a dominating set.
The difference between $\varepsilon_{t, i}$ and $o_{t, i}$ is key to take advantage of the graph structure. First, we need to lower bound $o_{t, i}$ to ensure that
enough observations (counted by $O_{t, i}$ in Algorithm~\ref{alg:exp3G++}) are made for each arm, such that our gap estimates are reliable.
Simultaneously, we upper bound $\varepsilon_{t, i}$ to ensure that the extra exploration is not too costly. Here we benefit from the fact that $S_t$ is a strongly independent set on $G$.

We ensure that all arms get sufficiently many observations and derive the following concentration bounds on the gap estimates $\hat\Delta_{t, i}$ computed by Algorithm \ref{alg:exp3G++}. Concentration of the gap estimates around the true gaps is crucial for bounding the regret in the stochastic setting.
% Note that tuning the exploration according to $\tilde \alpha$ is not necessary for the concentration bounds on the gap estimates to hold. We derive the following general result. 

% 
% \begin{lemma}
% \label{lem:delta_upperbound}
% For all $i \in [K]$ and $t \geq 1$,
% %
% ${\displaystyle
%     \prob[\hat\Delta_{t, i} \geq \overline\Delta_i] \leq \frac{1}{Kt^{\gamma - 1}}
% }$.
% \end{lemma}

% \begin{lemma}
% \label{lem:delta_lowerbound}
% For all $i \in [K]$ such that $\Delta_i > 0$, $c \in \lrs{1, K}$
% and for all
% \[
%     t \geq \tmin(i) := \max_{t\ge 0} \left\{\frac{1}{2}\sqrt{\frac{c \ln K}{t K^2}} \leq \frac{ \beta \ln t}{t\Delta_i^2}\right\}~.
% \]
% If Algorithm \ref{alg:exp3G++} is run with $\gamma \geq 3$ and $\beta \geq 64(\gamma+ 1)$ and any exploration rates such that~\eqref{eq:oti_lb} holds, then:
% %
% \begin{equation}
%     \prob[\hat\Delta_{t, i} \leq \frac{1}{2}\Delta_i] \leq \lr{\frac{\ln t}{t\Delta_i^2}}^{\gamma - 2} + \frac{2}{Kt^{\gamma - 1}} + 2 \lr{\frac{1}{t}}^{\frac{\beta}{10}}. 
% \end{equation}
% \end{lemma}

\begin{lemma}
\label{lem:delta_bounds} If Algorithm \ref{alg:exp3G++} is run with parameters $\gamma \ge 3$, $\beta \ge 64(\gamma + 1) \ge 256$, and exploration rates $\varepsilon_{t, i}$, such that for all $t \ge 1$ and $i \in V$, $P_{t, i}$ satisfies equation \eqref{eq:oti_lb}
% $o_{t, i} = \min \lrc{\frac{1}{2K}, \frac{1}{2} \sqrt{\frac{\lambda \ln K}{tK^2}},  \frac{\beta \ln t}{t \hat\Delta_{t, i}^2}}$
for some
% chosen
$\lambda \in [1, K]$, then for all $i \in V$ and $t \geq 1$,
%
% ${\displaystyle
%     \prob[\hat\Delta_{t, i} \geq \overline\Delta_i] \leq \frac{1}{Kt^{\gamma - 1}}
% }$.
\[
 \prob[\hat\Delta_{t, i} \geq \overline\Delta_i] \leq \frac{1}{Kt^{\gamma - 1}}.
\]
Furthermore, for any arm $i$ with $\Delta_i > 0$ let
${
   \tmin(i) :=  \max \left\{t\ge 0: \frac{1}{2}\sqrt{\frac{\lambda \ln K}{t K^2}} \leq \frac{ \beta \ln t}{t\Delta_i^2}\right\}
}$.
Then for any arm $i$ with $\Delta_i > 0$ and $t \geq \tmin(i)$,
\begin{equation}
    \prob[\hat\Delta_{t, i} \leq \frac{1}{2}\Delta_i] \leq \lr{\frac{\ln t}{t\Delta_i^2}}^{\gamma - 2} + \frac{2}{Kt^{\gamma - 1}} + 2 \lr{\frac{1}{t}}^{\frac{\beta}{10}}. 
\end{equation}
\end{lemma}
A proof of the lemma is provided in \Cref{appen:prop_gaps}. 

% In the stochastic regime, we need to control the gap estimates and tune the parameters of the confidence bound properly.
% To this end,
We run the algorithm with $\gamma = 4$ and $\beta = 64 (\gamma + 1) = 320$
% . Note that this 
which is a different parameterization  from the EXP3++ algorithm \citep{SL17}, which uses $\gamma = 3$ and $\beta = 256$.
Picking a larger value of $\gamma$ means that the confidence intervals are slightly larger, which allows us to obtain a better dependency on the suboptimality gaps. % using the following Lemma.

Indeed, under the same assumptions as in \Cref{lem:delta_bounds}, if $\gamma = 4$ and $t \geq \tmin(i)$, we have that
\begin{equation*}
\frac{(\ln t)^2}{t} \leq \frac{\lambda \Delta_i^4 \ln K}{4K^2 \beta^2}
\;\text{, implying}\;
\lr{\frac{\ln t}{t\Delta_i^2}}^{2} = \frac{\lr{\ln t}^2}{t^2\Delta_i^4} = \frac{\lr{\ln t}^2}{t} \frac{1}{t\Delta_i^4} \leq  \frac{\lambda \Delta_i^4
\ln K }{4K^2 \beta^2} \frac{1}{t\Delta_i^4} = \frac{1}{t} \frac{\lambda \ln K}{4K^2\beta^2}. 
\end{equation*} 

\section{Adversarial Analysis}
\label{sec:adv}

Our result for the adversarial regime
% generalizes the analysis of \citet{ACBDK15} to time varying learning rates and 
is given in the following theorem.
\begin{theorem}\label{th:adv}
% Let $G = (V, E)$ be a graph with . We consider the online learning problem induced by $G$.
Assume that Algorithm~\ref{alg:exp3G++} is run with a directed feedback graph $G = (V,E)$, a learning rate 
% $\eta_t = \frac{1}{4} \sqrt{\frac{1}{\alpha t}}$
$\eta_t = \sqrt{\frac{\ln K}{2 \tilde \alpha t}}$ and the exploration rate defined in~\eqref{eq:def_eti}--\eqref{eq:def_xiti} with $\lambda = \tilde\alpha$,
$\gamma = 4$, and $\beta = 320$. % where $\xi_{t, i} \geq \frac{4}{t^2}$ for all $t, i$.
Then the pseudo-regret against any oblivious loss sequence satisfies
$
    \mathcal{R}_T \leq  4 \sqrt{\tilde \alpha T \ln K} + K
$,
where $K = |V|$ and $\tilde \alpha$ is the strong independence number of $G$.
\end{theorem}

On undirected graphs, this bound matches the one of
% \citet{ACBDK15}
\citet{ACBGMMS14}  which implies that in the adversarial regime we are not paying a price for the extra guarantees that we derive in the stochastic regime. % This rate is near optimal, matching the $\Omega\big(\sqrt{\alpha(G)T}\big)$ lower bound within factors logarithmic in $K$ \citep{ACBGMMS14}.
On directed graphs, if the difference between $\alpha$ and $\tilde \alpha$ is large, it may be advantageous to express the bound in terms of the independence number of the graph rather than its strong independence number at the cost of a logarithmic factor. See Section \ref{sec:time_var} for additional results in this direction.

We give a sketch of the proof here and defer the detailed proof to Appendix \ref{appen:adv}.
\paragraph{Proof sketch.}
We separate the first $K$ rounds, in which the algorithm plays deterministically, from the remaining rounds, where we bound separately the contributions to the regret from the exponential weights and from the extra exploration.
To bound the contribution of the extra exploration, we use that $\varepsilon_{t, i} \leq \frac{1}{2} \sqrt{\frac{\tilde \alpha \ln K}{t K^2}}$ for all $t$ and $i$, meaning that the extra exploration contributes at most $O(\sqrt{\tilde \alpha T \ln K})$ to the regret. For bounding the contribution of the exponential weights to the regret, we follow the standard analysis of EXP3 with time varying learning rate \citep{BCB12}. 
We bound the second order term by exploiting the fact that $p_{t, i}$ and $q_{t, i}$ are close to each other because $\varepsilon_{t, i} \leq \frac{1}{2K}$ for all $t$ and $i$. 
% This allows us to use the following lemma, implying that the exponential weights contribute at most $O(\sqrt{\alpha T}\ln (KT))$ to the regret.
This allows us to bound the second order term in terms of $\sum_{i \in V} \frac{p_{t, i}}{P_{t, i}}$, which is the sum on the ratios between the probability of playing an arm and the probability of observing it. This quantity can be bounded by the strong independence number of the graph (Lemma 10 \citet{ACBGMMS14}).
% but also by the independence number at the cost of a multiplicative logarithmic factor (Lemma 5 \citet{ACBDK15}). The first bound is always tighter for undirected graphs, but the second can be useful for undirected graphs.
%
% \begin{lemma}[\protect{\citep[Lemma 5]{ACBDK15}}] \label{imp:ACBDK15_lem5_body}
% Let $G = (V, E)$ be a directed graph with $|V| = K$, in which each node $i \in V$ is assigned a positive weight $w_i$. Assume that $\sum_{i \in V} w_i \leq 1$, and that $w_i \geq \epsilon$ for all $i \in V$ for some constant $0 < \epsilon < \frac{1}{2}$. Then
% \[ \sum_{i \in V} \frac{w_i}{w_i + \sum_{j \in \Nin(i)} w_j} \leq 4 \alpha \ln \lr{\frac{4K}{\alpha \epsilon}}, 
% \]
% where $\alpha$ is the independence number of $G$.
% \end{lemma}

%

\section{Stochastic Analysis}
\label{sec:stoc}
Our result for the stochastic regime is given in the following theorem.
% This allows us to prove the following result.
%

\begin{theorem}\label{th:stoc}
Let $G = (V, E)$ be a directed feedback graph with $K = |V|$ and a strong independence number $\tilde \alpha$.
Under the same conditions as in  Theorem~\ref{th:adv},
% , and with $\xi_{t,i}$ defined in~\eqref{eq:def_xiti}
the pseudo-regret of Algorithm~\ref{alg:exp3G++}
% , run with parameters $\gamma = 4$ and $\beta = 320$
against any stochastic stochastic loss sequence, satisfies:
\begin{align*}
    \mathcal{R}_T \leq \max_{\mathrm{Ind} \in \mathcal{I}(G)} \!\! \lrc{ \sum_{i \in \mathrm{Ind} \,:\, \Delta_i > 0} \!\!\!\! \frac{4\beta \lr{\ln T}^2}{\Delta_{i}}}
    +   2 \tilde \alpha\ln T
    + \frac{81 \beta K}{\Dmin^2} \ln\lr{\frac{6\beta^2 K^2}{\Dmin^4}}
    \!\! + \!\! \sum_{i \,:\, \Delta_i > 0} \frac{4 \tilde \alpha}{\Delta_i}.
\end{align*} 
\end{theorem}
%
% One key property of this result is that the terms that depend on all arms are time independent, and the only terms that scale with $T$ are in an independence set.
% One key property of this result is
We remark that the last two terms do not depend on $T$. Moreover, the leading coefficient of the term scaling with $(\ln T)^2$ sums over an independence set (as opposed to summing over the entire action set).
The lower bound for this problem scales as  $\Omega (c^* \ln T)$, where $c^*$ is a graph dependent quantity which takes the size of the suboptimality gaps into account \citep{BES14}.  Compared to that, our result is suboptimal by a logarithmic factor and our dependency on the strong independence number of $G$ is weaker. 
The algorithms of \citet{BES14, BLES17} and \citet{CHK16} almost match the lower bound, but their elimination based structure prevents them from being applicable in best-of-both-worlds settings. 
In the undirected case, we obtain the same dependence on $T$ and on the set of arms as the UCB-N algorithm analysed by \citet{LTW20}.
%
% The proof follows from decomposing the pseudo-regret into: (1) the initial rounds, (2) the contribution of the exponential weights, (3) the contribution of the extra exploration, and then bounding each part separately.

% On some problem instances, the dependency on the worst case independent set maximizing $\sum_{i \in \mathrm{Ind}} \Delta_i^{-1}$ may be a rather loose bound.

We provide a sketch of the proof here. The detailed version can be found in Appendix \ref{appen:stoch}.

\paragraph{Proof sketch.}
% Without loss of generality, we assume that the vertices are sorted so that $\Delta_1 \geq \Delta_2 \geq \cdots \ge \Delta_K = 0$. 
Let $\tmin = \max_{i: \Delta_{i} > 0}  \lrc{\tmin(i)} =  \max \left\{ {t\ge 0} : \frac{1}{2}\sqrt{\frac{\tilde \alpha \ln K}{t K^2}} \leq \frac{ \beta \ln t}{t\Dmin^2}\right\} $.
The pseudo-regret can be decomposed by treating the first $\tmin$ rounds like in the adversarial case, and by using a refined bound for the stochastic regime in the remaining rounds.
\begin{align*}
    R_T   = R_{\tmin} +  \sum_{t = \tmin}^T \sum_{i \,:\, \Delta_i > 0} \Delta_i \E[p_{t, i}]
    \leq  R_{\tmin} + \sum_{t = \tmin}^T  \sum_{i \,:\, \Delta_i > 0} \Delta_i  \big(\E[q_{t, i}] +   \E[\varepsilon_{t, i}]\big).  \numberthis{} \label{decomp_stoch_reg2}
\end{align*}
Note that $\tmin$ is time independent: $\tmin = \frac{c}{\Dmin^4} \lr{\ln \lr{\frac{c}{\Dmin^4}}}^2$ for a positive constant $c$, therefore, 
\[R_{\tmin} = C_0 \sqrt{\tilde \alpha \ \tmin}  = C_1 \frac{K}{\Dmin^2} \lr{\ln\lr{\frac{K}{\Dmin}}},\]
where the first equality follows from Theorem~\ref{th:adv} and $C_0, C_1$ are universal constants. After the initial $\tmin$ rounds, enough observations on all arms have been gathered to ensure with high probability that the gap estimates of all arms are close to their true gaps, as stated in Lemma~\ref{lem:delta_bounds}.
% and~\ref{lem:delta_lowerbound}
These concentration inequalities allow us to show that the two following propositions hold.

\begin{prop}[informal]
\label{prop:qti}
The contribution of the exponential weights to the pseudo-regret can be bounded as:
% \begin{equation*}
% \sum_{i \,:\, \Delta_i > 0} \Delta_i \sum_{t = \tmin(i)}^T  \E[q_{t, i}] \leq \ln T \lr{\frac{\alpha \ln K}{4K\beta^2} + 1} + \sum_{i \,:\, \Delta_i > 0} \frac{128 \alpha}{\Delta_i}.
% \end{equation*}
\begin{equation*}
\sum_{t = \tmin}^T \sum_{i \,:\, \Delta_i > 0} \Delta_i  \E[q_{t, i}] \leq C_2\sum_{i \,:\, \Delta_i > 0} \frac{\tilde \alpha}{\Delta_i} + O\lr{\tilde \alpha \ln T}
\end{equation*}
for a universal constant $C_2$.
\end{prop}

\begin{prop}[informal]
\label{prop:eti}
The contribution of the extra exploration to the pseudo-regret can be bounded as:
% \begin{equation*}
%     \sum_{i \,:\, \Delta_i > 0} \Delta_i \sum_{t = \tmin(i)}^T  \E[\varepsilon_{t, i}] \leq \max_{\mathrm{Ind} \in \mathcal{I}(G)} \lrc{ \sum_{i \in \mathrm{Ind} \,:\, \Delta_i > 0} \frac{4\beta \ln^2 T}{\Delta_{i}}} + \frac{\alpha  \ln T \ln K}{4K\beta^2} + 12K + 3.
% \end{equation*}
\begin{equation*}
    \sum_{t = \tmin}^T \sum_{i \,:\, \Delta_i > 0} \Delta_i \E[\varepsilon_{t, i}] = O\lr{ \max_{\mathrm{Ind} \in \mathcal{I}(G)} \lrc{ \sum_{i \in \mathrm{Ind} \,:\, \Delta_i > 0} \frac{\ln^2 T}{\Delta_{i}}} %+ K}
    + \tilde \alpha  \ln T}.
\end{equation*}
% for some universal constants $C_3$.
\end{prop}
Formal statements and proofs of the above propositions are in Appendix~\ref{appen:stoch}. These propositions ensure that after $\tmin$ steps the exponential weights of all suboptimal arms $i$ are small, the extra exploration $\varepsilon_{t,i}$ achieves the correct rate, and that the sum of the probabilities that the suboptimality gap estimates fail in any of the rounds is of order $O( \tilde \alpha \ln T)$.
Applying these propositions to \Cref{decomp_stoch_reg2} finishes the proof.

% The intuition behind the proof is that the exploration term $\varepsilon_t$ allows to gather enough data to keep reliable gap estimates, while ensuring after $\tmin(i)$ steps that the gap estimates $\hat \Delta_{t, i}$ are lower bounded. Then,

Our approach to bound the pseudo-regret in the initial rounds differs from the one of  \citet*{SL17} as we take advantage of the adversarial bound in these rounds. (\citet*{MG19} used a similar approach to derive best-of-both worlds guarantees for the Hedge algorithm.)
This refinement improves upon the result of \citet*{SL17} by replacing $\sum_{i: \Delta_i > 0} \frac{1}{\Delta_i^3}$ with $\frac{1}{\Dmin^2}$ (numerical constants ignored) in the time-independent part of the bound.
\section{Extension to Time Varying Feedback Graphs}
\label{sec:time_var}

A natural extension of our results is to the setting where the feedback graphs are allowed to change over time. We consider a setting, where an oblivious adversary chooses a feedback graph at each round and the algorithm observes the graph at the beginning of the round.
In the stochastic regime, the knowledge of the full feedback graph is required at the beginning of the round in order to construct the exploration set. Since computation of the independence number % for each one of these graphs
is an NP-hard problem,
% meaning that it cannot be done in polynomial time of the number of vertices $K$ in the graph.
it is vital that we tune our algorithm so that it can adapt to the independence number of the graphs without needing to actually compute it. 
In order to do so, we define the quantity:
\begin{equation}
    \theta_t := \sum_{i \in V} \frac{p_{t, i}}{P_{t, i}}, \label{eq:def_theta}
\end{equation}
which is the sum of the ratios of the probability of playing an arm to the probability of observing its loss.
By adapting the learning rate we show that in the adversarial setting our regret bound is of order $\sqrt{\sum_{t=1}^T \theta_t}$, for which we can use Lemma 5 of \citet{ACBDK15} to upper bound $\theta_t$ in terms of $\alpha_t$ to obtain the desired result. 
% We also redefine the exploration $\varepsilon_{t}$ by reducing its second term.
% \begin{equation}
%     \varepsilon_{t, i} = \min \lrc{\frac{1}{2K}, \frac{1}{2} \sqrt{\frac{\ln K}{t K^2}}, \xi_{t, i}}, \label{eq:def_eti_var}
% \end{equation}
As we do not know the independence numbers ahead of time, we tune the exploration rates defined in equation \eqref{eq:def_eti} with $\lambda = 1$ to ensure that the exploration is never too large.
This exploration rate allows us to apply Lemma \ref{lem:delta_bounds} with $\lambda = 1$, and derive the following result.

\begin{theorem}
\label{th:adv_time_var}
Assume that Algorithm~\ref{alg:exp3G++} is run on a sequence of arbitrarily generated feedback graphs $G_1, G_2, \dots$
with learning rate $\eta_t = \sqrt{\frac{\ln K}{2\sum_{s = K}^{t-1}\theta_s}} $ and exploration rates defined in~\eqref{eq:def_eti} and \eqref{eq:def_xiti} with $\lambda = 1$, $\gamma = 4$ and $\beta = 320$. 
Then the pseudo-regret against any oblivious loss sequence satisfies
\[
    \mathcal{R}_T \leq   9 \sqrt{\ln K} \sqrt{\ln \lr{KT}} \sqrt{\sum_{t = 1}^T \alpha_t}  + 2K,
\]
where for all $t \ge 1$, $\alpha_t$ is the independence number of $G_t$. Simultaneously, the pseudo-regret against stochastic losses satisfies:
 \begin{align*}
     R_T \leq &  
     \inf_{0 \leq n \leq T} \lrc{
     \max_{S \subset V: |S| = \tilde \alpha_n
     } \lrc{ \sum_{i \in S: \Delta_i > 0} \frac{4\beta \ln^2 T}{\Delta_{i}}} + n} \\
     & + 2 \ln T  +  \sum_{i: \Delta_i > 0} \frac{16 K}{\Delta_i} + \frac{181 \beta K^{3/2}}{\Dmin^2} \lr{\ln\lr{\frac{20\beta^2 K^{5/2}}{\Dmin^4}}}^2,
\end{align*}
where
% , for all $t \ge 1$, $\alpha_t$ is the independence number of $G_t$ and
$\tilde \alpha_n$ is the $n^{th}$ largest element in the set containing the strong independence number of all the $G_t$ for $t \leq T$.
\end{theorem}
A proof of this theorem is provided in Appendix~\ref{appen:sec_var}.
In the adversarial regime, adapting to graphs that change over time replaces the dependence on the strong independence number by the average on the independence numbers of the sequence of graphs, but comes at the cost of a multiplicative logarithmic factor. In the stochastic regime, the first term of the bound is a sum over the $\tilde \alpha_n$ arms that have the smallest non-zero suboptimality gaps. In the case of undirected graphs, if we upper bound the infimum  by taking $n = 0$, we have $\tilde \alpha_0 = \max_{t > 1} \lrc{\alpha(G_t)}$, which matches the dependency on gaps achieved by \citet{CHK16}, who got an $O \lr{ \max_{S \subset V\backslash \lrc{i^*}: |S| = O\lr{\alpha}} \sum_{i \in S} \frac{\ln T}{\Delta_i}}$ bound.

% \paragraph{Run-time complexity}
% % these are currently notes
% We assume that the learner is given the independence number $\alpha_t$ along graph $G_t$. In case it is not given as an input, finding the maximum set cover of $G_t$ is an NP-hard problem. A simple algorithm to find a maximum set cover has a run-time complexity of $O\lr{K^2 2^K}$.  In comparison, a naive implementation of Algorithm \ref{alg:Explore_set} runs in $O\lr{K^3}$.
% In case we decide to compute this exactly, then we could change our greedy exploration set procedure for a method that could take advantage of the trade-off between the cost of playing an arm and the amount of observations that can be gathered with this arm. 

\section{Conclusion}
\label{sec:conclusion}
% https://youtu.be/9cQgQIMlwWw
\citet{EK21} 
% qualified best-of-both-worlds online learning with feedback graphs to be a challenging task and 
left open the following questions: is it possible to achieve best-of-both-worlds regret bounds in terms of the independence number, and can the dependence on $T$ in their regret bounds be improved?
We partially answered these questions with the EXP3.G++ algorithm and derived near-optimal best-of-both-worlds guarantees for directed feedback graphs. Our regret bounds depend on the independence number of the feedback graphs and improve upon the results of \citet{EK21} by poly-logarithmic factors in both the adversarial and stochastic regimes. 
Furthermore, we extended our results to time-varying feedback graphs with a computationally efficient algorithm.

% We have presented the EXP3.G++ algorithm, and shown that it is capable of achieving near optimal best-of-both worlds guarantees, by only being suboptimal by a logarithmic factor in both the stochastic and the adversarial regime. It is the first algorithm capable of obtaining best-of-both worlds guarantees that depend on the independence number of the graph rather than on the clique covering number, and to adapt to arbitrary sequences of graphs. 
% EXP3++ computationally efficient as computing the exploration set is done in polynomial time of the number of arms. 

\begin{ack}
CR and YS acknowledge partial support by the Independent
Research Fund Denmark, grant number 9040-00361B. DvdH and NCB gratefully acknowledge partial support from the MIUR PRIN grant Algorithms, Games, and Digital Markets (ALGADIMAR) and the EU Horizon 2020 ICT-48 research and innovation action under grant agreement 951847, project ELISE (European Learning and Intelligent Systems Excellence).
\end{ack}

% \section*{References}

% \medskip
\DeclareRobustCommand{\VAN}[3]{#3} % proper Dutch 'van/de' capitalisation
\bibliography{references}
\DeclareRobustCommand{\VAN}[3]{#2} % proper Dutch 'van/de' capitalisation

\appendix

% \section{Appendix}
\label{appendix}
\section{Tools to Bound Series}
\label{appen:tools}
We use the following lemmas to bound series.
\begin{lemma}[Lemma 11 \citep{SL17}]
\label{imp:SL17_lem11}
For $\gamma \geq 2$ and $m \geq 1$:
\[ \sum_{k = m}^n \frac{1}{k^\gamma} \leq \frac{1}{2 m^{\gamma -1}}.
\]
\end{lemma}

\begin{lemma}[Lemma 8 \citep{SBCA14}]
\label{imp:SBCA14_lem8}
For any sequence of non-negative numbers $a_1, a_2, \dots, $ such that $a_1 > 0$ and any power $\gamma \in (0, 1)$ we have:
\[\sum_{t= 1}^T \frac{a_t}{\lr{\sum_{s = 1}^t a_s}^\gamma} \leq \frac{1}{1 - \gamma} \lr{\sum_{t = 1}^T a_t}^{1-\gamma}. \]
\end{lemma}

We also require a variation of this bound to handle the case where the denominator of the sum only sums up to index $t - 1$. The proof of this Lemma follows from \citet[Lemma 14]{GSE14} that we generalized to adapt to sequences of $a_t$ that are not restricted to the $[0, 1]$ interval. 

\begin{lemma}
\label{lem:sum_offset}
For any sequence $a_1, a_2, \dots, $ such that $\alpha_s \in [1, K]$ for all $s$, we have:
\begin{equation*}
    \sum_{t = 1}^T \frac{a_t}{\sqrt{ K +\sum_{s < t} a_s}}\leq 2 \sqrt{\sum_{t = 1}^T a_t} + \sqrt{K}.
\end{equation*}
\end{lemma}

\begin{proof}
Let $s_t = \sum_{n = 1}^t a_t$, and define $s_0 := 0$. We want to bound $\sum_{t = 1}^T \frac{a_t}{\sqrt{ K +\sum_{s < t} a_s}} = \sum_{t = 1}^T \frac{a_t}{\sqrt{K + s_{t-1}}}$, where $\frac{1}{\sqrt{K + s}}$ is a decreasing function of $s$. Thus we have:

\begin{align*}
\sum_{t = 1}^T \frac{a_t}{\sqrt{K + s_{t-1}}}
& = \sum_{t = 1}^T \frac{a_t}{\sqrt{K + s_{t}}} + \sum_{t = 1}^T a_t \lr{\frac{1}{\sqrt{K + s_{t-1}}}  - \frac{1}{\sqrt{K + s_{t}}}} \\
& \leq \sum_{t = 1}^T \frac{a_t}{\sqrt{s_{t}}} + K \sum_{t = 1}^T\lr{\frac{1}{\sqrt{K + s_{t-1}}}  - \frac{1}{\sqrt{K + s_{t}}}} \\
& \leq \sum_{t = 1}^T \frac{a_t}{\sqrt{s_{t}}} + K \frac{1}{\sqrt{K + s_{0}}} \\
& \leq 2 \sqrt{s_{t}} + \sqrt{K},
\end{align*}
where we use Lemma \ref{imp:SBCA14_lem8} in the last step.
\end{proof}

\begin{lemma}[Lemma 3 \citep{TS18}]
\label{imp:TS18_lem_3}
For $c > 0$ we have
\[ \sum_{t = 1}^\infty e^{-c\sqrt{t}} \leq \frac{2}{c^2}  \qquad  \text{and} \qquad \sum_{t = 1}^\infty e^{-ct} \leq \frac{1}{c}. 
\]
\end{lemma}

\section{Proof of Theorem \ref{th:adv}}
\label{appen:adv}
We follow the proof structure of Theorem 2 from \citet{ACBDK15}, and use Lemma 7 from \citet{SS14} where $X_{t, i} = \tilde \ell_{t, i}$ for all $t, i$ as a base for the analysis of EXP3.

\begin{lemma}[Lemma 7 \citep{SS14}]\label{imp:SS14_lem7} For any $K$ sequences of non-negative numbers $X_{1, i}, X_{2, i}, \dots$ indexed by $i \in [K]$, and any non-increasing positive sequence $\eta_1, \eta_2, \dots$, for $q_{t, i}=\frac{\exp(-\eta_t \sum_{s = 1}^{t-1} X_{s, i})}{\sum_{j \in [K]}\exp(-\eta_t \sum_{s = 1}^{t-1} X_{s, j})}$ (assuming for $t = 1$ the sum in the exponent is $0$) we have:
\[ \sum_{t = 1}^T \sum_{i = 1}^K q_{t, i} X_{t, i} - \min_{k \in [K]} \sum_{t = 1}^T X_{t, k} \leq \frac{\ln K}{\eta_T} + \sum_{t = 1}^{T} \frac{\eta_t}{2}\lr{\sum_{i \in [K]} q_{t, i}X_{t, i}^2}.\]
\end{lemma} 

We also require Lemma 10 from \citet{ACBGMMS14} to take advantage of the structure of the feedback graph. That Lemma depends on a graph dependent quantity: the maximum acyclic subgraph of a feedback graph $G$, which is defined  by \citeauthor{ACBGMMS14} as follows. For undirected graphs, using this approach leads to tighter bounds compared to using the bound from \citet{ACBDK15}.

\begin{definition}
\label{def_mas}
Given a directed graph $G = (V, E)$, an acyclic subgraph of $G$ is any $G' = (V', E')$ such that $ V' \subseteq V$ and $E' = E \cap (V' \times V')$, with no (directed) cycles. 
We denote by $\mas(G) = |V'|$ the maximum size of such a $V'$. 
\end{definition}
A key property of the maximum acyclic subgraph is that for any graph $G$, $\alpha (G) \leq \mas(G)$ and for undirected graphs, $\alpha(G) = \mas(G)$ \citep{ACBGMMS14}. We now show that for any directed graph $G$, the maximum acyclic subset of $G$ can be upper bounded by its strong independence number. 

\begin{prop}
\label{prop:mas_strong_indep}
Let $G = (V, E)$ be a directed graph. $\mas(G) \leq \tilde \alpha (G)$. 
\end{prop}

\begin{proof}
Let $G' = (V', E')$ be an acyclic subgraph of $G$, where $ V' \subseteq V$ and $E' = E \cap (V' \times V')$.
For any $i, j \in V'$, we know that $(i, j) \not\in E'$ or  $(j, i) \not\in E'$, otherwise $i$ and $j$ would be part of a cycle which contradicts the definition of $G'$. Thus $i$ and $j$ are strongly independent and $V'$ is a strongly independent set.
As this holds for all acyclic subgraphs of $G$, we deduce that $\mas(G) \leq \tilde \alpha(G)$ which finishes the proof.
\end{proof}

% We can use the next result by upper bounding $\mas(G) \leq \tilde \alpha(G)$.
This characterization allows us to use the following lemma and derive bounds that scale with the strong independence number.

\begin{lemma}[Lemma 10 \citet{ACBGMMS14}]
\label{imp:ACBGMMS14_lem10}
let $G = (V, E)$ be a directed graph with vertex set $V = \lrc{1, \dots, K}$, and arc set $V$. Then, for any distribution $p$ over $V$ we have:
\[ \sum_{i = 1}^K \frac{p_i}{p_i + \sum_{j \in \Nin(i)} p_j} \leq \mas (G).
\]
\end{lemma}

With those results, we can move on to the proof of Theorem \ref{th:adv}.
\begin{proof}[Proof of \Cref{th:adv}] \label{proof:th_adv}
Without loss of generality, we assume that $K \geq 2$.

Recall that the algorithm initializes by playing each arm once, which adds at most $K$ to the regret. The EXP3 part of the analysis starts from round $K + 1$. 
We can upper trivially upper bound the first $K$ rounds by $1$ and then analyse the algorithm from round $t = K+1$. Precisely, we bound the pseudo-regret as:
\begin{align*}
    \mathcal{R}_T & =   \E[\sum_{t = 1}^T \ell_{t, I_t}] -  \min_i\E[ \sum_{t = 1}^T\ell_{t, i}] \\
    & \leq K +  \E[\sum_{t = K+1}^T \ell_{t, I_t}] -  \min_i\E[ \sum_{t = K+1}^T\ell_{t, i}]\\
     & = K + \E[ \sum_{t = K+1}^T \sum_{i = 1}^K p_{t, i} \E_t\lrs{\tilde \ell_{t, i}} -  \sum_{t = K+1}^T \E_t\lrs{\tilde \ell_{t, i^*}}] \\
    & \leq K + \E[ \sum_{t = K+1}^T \sum_{i = 1}^K q_{t, i} \E_t\lrs{\tilde \ell_{t, i}} -  \sum_{t = K+1}^T \E_t\lrs{\tilde \ell_{t, i^*}}] + \E[ \sum_{t = K+1}^T \sum_{i = 1}^K \varepsilon_{t, i} \E_t\lrs{\tilde \ell_{t, i}} ],\numberthis{} \label{proof_adv:reg_decomp}
\end{align*} where $ i^* = \argmin \lrc{\sum_{t = K+1}^T \E_t\lrs{\tilde \ell_{t, i^*}}}$, and $\E_t\lrs{\tilde \ell_{t, i}} = \ell_{t, i}$.  Equation \eqref{proof_adv:reg_decomp} follows from  $p_{t, i} \leq q_{t, i} + \varepsilon_{t, i}$.  We can consider the contribution of $q_{t, i}$ and $\varepsilon_{t, i}$ separately.

Recalling that $\tilde L_{K} = 0$ as the unbiased loss estimates are only updated from round $K+1$, we can apply Lemma \ref{imp:SS14_lem7}  to the first expectation in equation \eqref{proof_adv:reg_decomp}, and we get:

\begin{align*}
     \E[ \sum_{t = K+1}^T \sum_{i = 1}^K q_{t, i} \E_t\lrs{\tilde \ell_{t, i}} -  \sum_{t = K+1}^T \E_t\lrs{\tilde \ell_{t, i^*}}] 
     &\leq \frac{\ln K}{\eta_T} + \sum_{t = K+1}^{T} \frac{\eta_t}{2} \E[\sum_{i \in V} q_{t, i}\E_t\lrs{\tilde \ell_{t, i}^2}] \\
     &\leq \frac{\ln K}{\eta_T} + \sum_{t = K+ 1}^{T} \frac{\eta_t}{2} \E[\sum_{i \in V} \frac{q_{t, i}}{P_{t, i}}], \numberthis{} \label{proof_adv:qtiPti}
\end{align*}
where the last step uses
$\E_t\lrs{\tilde \ell_{t, i}^2} = \E_t\lrs{\frac{\ell_{t, i}^2}{P_{t, i}^2} \1[i \in \Nout(I_t)]} \leq  \E_t\lrs{\frac{1}{P_{t, i}^2} \1[i \in \Nout(I_t)]} = \frac{1}{P_{t, i}}.$

In order to bound the sum, we use Lemma \ref{imp:ACBGMMS14_lem10} with $\tilde \alpha(G) \geq \mas(G)$.
% \ref{imp:ACBDK15_lem5} with $p_{t, i} \geq \frac{4}{t^2} \geq \frac{4}{T^2}$, which is a lower bound on the exploration induced by mixing in $\varepsilon_{t}$ in $p_{t}$.
Using $\frac{1}{2K}$ as an upper bound on $\varepsilon_t$, we ensure that for all $t$ and $i$, $p_{t, i} \geq (1 - \sum_{j \in V} \varepsilon_{t, j})q_{t, i} \geq \frac{1}{2}q_{t, i}$ which gives:

% \[  \sum_{i \in V}\frac{q_{t, i}}{P_{t, i}} \leq 2 \sum_{i \in V} \frac{p_{t, i}}{P_{t, i}}  \leq 8\alpha \ln \lr{\frac{KT^2}{\alpha}} \leq 16 \alpha \ln \lr{KT}.\]
\[  \sum_{i \in V}\frac{q_{t, i}}{P_{t, i}} \leq 2 \sum_{i \in V} \frac{p_{t, i}}{P_{t, i}}  \leq 2 \tilde \alpha. \] 
Moving on to the contribution of $\varepsilon_t$:

\begin{align*}
    \E[ \sum_{t = K+1}^T \sum_{i = 1}^K \varepsilon_{t, i} \E_t\lrs{\tilde \ell_{t, i}} ] & = \sum_{t = K+1}^T  \E[ \frac{1}{2} \sqrt{\frac{\tilde \alpha \ln K}{tK^2}}\sum_{i = 1}^K  \ell_{t, i} ] \\
    & = \sum_{t = 1}^T \frac{1}{2} \sqrt{\frac{\tilde \alpha \ln K}{t}}.
\end{align*}

Plugging those two bounds in Equation \eqref{proof_adv:reg_decomp}, we obtain the following regret bound:

% \[ \mathcal{R}_T \leq  K + \frac{\ln K}{\eta_T} + \sum_{t = 1}^T  \frac{1}{2} \sqrt{\frac{\alpha \ln K}{t}} + \sum_{t = 1}^T   8 \ \eta_t \ \alpha \ln \lr{KT}.  \]

\[ \mathcal{R}_T \leq  K + \frac{\ln K}{\eta_T} + \sum_{t = 1}^T  \frac{1}{2} \sqrt{\frac{\tilde \alpha \ln K}{t}} + \sum_{t = 1}^T  \eta_t \ \tilde \alpha.  \]

% Finally, we use $\eta_t = \frac{1}{4}\sqrt{\frac{1}{\alpha t}}$ and that $\sum_{t = 1}^T \frac{1}{\sqrt{t}} \leq 2\sqrt{T}$ to get:
% \[ \mathcal{R}_T \leq K + 4 \sqrt{\alpha T}\ln (K) +  \sqrt{\alpha T \ln K} +\frac{16}{4} \sqrt{\alpha T} \ln \lr{KT} \leq  9 \sqrt{\alpha T} \ln \lr{KT} + K . \]

Finally, we use $\eta_t = \sqrt{\frac{\ln K}{2 \tilde \alpha t}}$ and that $\sum_{t = 1}^T \frac{1}{\sqrt{t}} \leq 2\sqrt{T}$ to get:
\[ \mathcal{R}_T \leq K + \sqrt{2 \tilde \alpha T \ln (K)} +  \sqrt{ \tilde \alpha T \ln K} + \frac{2}{\sqrt{2}} \sqrt{\tilde \alpha T \ln K} \leq  4 \sqrt{\tilde \alpha T \ln K} + K . \]
\end{proof}

\section{Properties of the Gaps Estimates}
\label{appen:prop_gaps}
In this section, we provide upper and lower high probability bounds for the estimates of the suboptimality gaps. We decompose the proof of Lemma \ref{lem:delta_bounds} in two parts.

\subsection{Upper bound}

We start by deriving a high probability upper bound. For this bound, we have to be careful with the fact that the gap estimates are clipped in the $[0, 1]$ interval. We first upper derive bounds on UCB and LCB.

\begin{lemma}
\label{lem:cb}
The confidence intervals satisfy:
\begin{align*}
    & \prob[\UCB_{t, i} \leq \mu_i] \leq \frac{1}{KT^{\gamma-1}} \\
	 \text{and} \qquad  & \prob[\LCB_{t, i} \geq \mu_i] \leq  \frac{1}{KT^{\gamma-1}}.
\end{align*}
\end{lemma}
\begin{proof}
Let $\overline{\UCB}_t$ and $\overline{\LCB}_t$ be the non clipped versions of the $\UCB_t$ and $\LCB_t$. In other words, for all $i$ and $t$:
\begin{align*}
    & \overline{\UCB}_{t, i} =  \frac{\hat L_{t-1, i}}{O_{t-1, i}} + \sqrt{\frac{\gamma \ln\lr{tK^{1/\gamma}}}{2O_{t-1, i}}} \\
	 \text{and} \qquad  & \overline{\LCB}_{t, i} =  \frac{\hat L_{t-1, i}}{O_{t-1, i}} - \sqrt{\frac{\gamma \ln\lr{tK^{1/\gamma}}}{2O_{t-1, i}}}.
\end{align*}
Then, through standard $\UCB$ analysis using Hoeffding's inequality (see for example \citet{SL17}), we have:
\begin{align*}
    & \prob[\overline{\UCB}_{t, i} \leq \mu_i] \leq \frac{1}{KT^{\gamma-1}} \\
	 \text{and} \qquad  & \prob[\overline{\LCB}_{t, i} \geq \mu_i] \leq \frac{1}{KT^{\gamma-1}}.
\end{align*}
By definition, we have $\UCB_{t, i} = \min\lrc{1, \overline{\UCB}_{t, i}} \leq \overline{\UCB}_{t, i}$ and $\LCB_{t, i} = \max\lrc{0, \overline{\LCB}_{t, i}} \geq \overline{\LCB}_{t, i}$, so:
\begin{align*}
    & \prob[\UCB_{t, i} \leq \mu_i] \leq \prob[\overline{\UCB}_{t, i} \leq \mu_i] \leq \frac{1}{KT^{\gamma-1}} \\
	 \text{and} \qquad  & \prob[\LCB_{t, i} \geq \mu_i] \leq \prob[\overline{\LCB}_{t, i} \geq \mu_i] \leq \frac{1}{KT^{\gamma-1}}.
\end{align*}
\end{proof}

Using this result, we can move on to bound the gap estimates.
\begin{proof}[Proof of the first part of \Cref{lem:delta_bounds}]
We recall that ${\hat \Delta_{t, i} = \max\lrc{0, \LCB_{t, i} - \min_{j \neq i} \UCB_{t, j}}}$. 
Then using \Cref{lem:cb}, we have:
\begin{align*}
    \prob[\hat \Delta_{t, i} \geq \overline{\Delta_i}]
    & = \prob[\LCB_{t, i} - \min_{j \neq i} \UCB_{t, j} \geq \overline{\Delta_i}] \\
    & \leq \prob[\LCB_{t, i} - \min_{j \neq i} \UCB_{t, j} \geq \Delta_i] \\
    & \leq \prob[\LCB_{t, i} \geq \mu_i] 
    + \sum_{j \neq i} \prob[ \UCB_{t, j} \leq \mu_j] \\
    & \leq K \frac{1}{Kt^{\gamma -1}} = \frac{1}{t^{\gamma -1}},
\end{align*}
where the first step takes advantage of the fact that $\overline \Delta_{i} > 0$ for all $i$, allowing to remove the maximum. The second step relies on $\Delta_i \leq \bar \Delta_i$, and we finish the proof with a union bound and applying Lemma \ref{lem:cb}.
\end{proof}

\subsection{Lower bound}
\label{sec:proof_lb}
To derive a lower bound on the gap estimates and prove the second part of Lemma \ref{lem:delta_bounds}, we start by proving some intermediate results.
recall that we use $o_{t, i}$ to lower bound the probability of observing the loss of arm $i$ at round $t$, and that by construction we have for all $t$, $i$:
 \[ o_{t, i} =\min \lrc{\frac{1}{2K}, \frac{1}{2}\sqrt{\frac{\lambda \ln K}{t K^2}}, \frac{\beta \ln t}{t \hat \Delta_{t, i}^2}}.\]
We also recall that for all $i$ such that $\Delta_i > 0$,  we defined $\tmin(i)$ as:
\[
\tmin(i) = \max \left\{t \ge 0: \ \frac{1}{2}\sqrt{\frac{\lambda \ln K}{t K^2}} \leq \frac{ \beta \ln t}{t\Delta_i^2}\right\}.
\]

\paragraph{A lower bound for $o_{t, i}$.} As $ \hat \Delta_{t, i}$ is a random variable, we derive a high probability lower bounds on $o_{t, i}$. 

\begin{definition}
We define the following events:
\begin{align*}
    \mathcal{E}(i, t) & = \lrc{\forall s \in \lrs{K+1, t}: o_{s, i} \geq \frac{\beta \ln t}{t \Delta_i^2}}, \\
    \mathcal{E}(i^*, i,  t) & = \lrc{\forall s \in \lrs{K+1,  t}: o_{s, i^*} \geq \frac{\beta \ln t}{t \Delta_i^2}}, 
\end{align*}
where $i^*$ is an optimal arm and $i$ a suboptimal arm.
\end{definition}
Note that the second event lower bounds the rate at which observations on optimal arm $i^*$ are gathered in terms of the gap with the suboptimal arm $i$.

\begin{lemma}
\label{lem:hp_lb_oti}
For any $i$ suboptimal arm and $i^*$ optimal arm, and $t \geq \tmin(i)$ and $\gamma \geq 3$, we have:
\begin{align*}
    \prob[\overline{\mathcal{E}(i, t)}] & \leq \lr{\frac{\ln t}{t \Delta_i^2}}^{\gamma - 2},\\
    \prob[\overline{\mathcal{E}(i^*, i,  t)}] & \leq \lr{\frac{\ln t}{t \Delta_i^2}}^{\gamma - 2}.
\end{align*}
\end{lemma}

\begin{proof}[Proof of Lemma \ref{lem:hp_lb_oti}]
The proof is very similar for the two inequalities.
By definition for all $s$ and $i$,  we have $\hat \Delta_{s, i} \leq 1$. Thus, $\frac{\beta \ln s}{s \hat \Delta_{s, i}^2} \geq \frac{\beta \ln s}{s}$. Then for $s \in \lrs{K+1, \frac{t \Delta_i^2}{\ln t}}$, we have $\frac{\beta \ln s}{s } \geq \frac{\beta \ln s\ln t}{t \Delta_i^2} \geq  \frac{\beta \ln t}{t \Delta_i^2}$, as $s > K \geq 2$, so $\ln s \geq 1$.
Furthermore, as $t \geq \tmin(i)$ then 
for all $ s \in \lrs{K+1,  t}$, we have $\frac{1}{2}\sqrt{\frac{\lambda \ln K}{s K^2}} \geq \frac{1}{2}\sqrt{\frac{\lambda \ln K}{t K^2}} \geq  \frac{\beta \ln t}{t \Delta_{i}^2}$ and $\frac{1}{2K} \geq   \frac{\beta \ln t}{t \Delta_{i}^2}$.
We deduce:
\begin{align*}
 \prob[\overline{\mathcal{E}(i, t)}] & = \prob[\exists s \in \lrs{\frac{t\Delta_i^2}{\ln t}, t} : o_{s, i} \leq \frac{\beta \ln t}{t \Delta_i^2}] \\
 & \leq \prob[\exists s \in \lrs{\frac{t\Delta_i^2}{\ln t}, t} : \hat \Delta_{s, i} \geq \Delta_i] \numberthis{} \label{proof:hp_lb_oti_1}\\
 & \leq \sum_{s = \frac{t\Delta_i^2}{\ln t}} \frac{1}{s^{\gamma -1}}\\ & \leq \frac{1}{2} \lr{\frac{\ln t}{t \Delta_i^2}}^{\gamma -2},
\end{align*}
where the last summation follows from Lemma \ref{imp:SL17_lem11}.
 The proof of the second inequality is similar, \Cref{proof:hp_lb_oti_1} only requiring the extra step:
 \[\prob[\exists s \in \lrs{\frac{t\Delta_{i^*}^2}{\ln t}, t} : \hat \Delta_{s, i} \geq \Delta_i] \leq  \prob[\exists s \in \lrs{\frac{t\Delta_{i^*}^2}{\ln t}, t} : \hat \Delta_{s, i} \geq \overline{\Delta_{i^*}}],\]
 which follows from $\Delta_i \geq \Dmin = \overline{\Delta_{ i^*}}$ as $i$ is a suboptimal arm and $i^*$ is an optimal arm.
 \end{proof}
 \paragraph{A lower bound for $O_{t, i}$} We now want to lower bound the number of observations of an arm up to round $t$.
 We rely on the following concentration inequality.

 \begin{theorem}[Theorem 8 \citep{SL17}]
 \label{th:concentration_ineq}
 Let $X_1, \dots, X_n$ be Bernoulli random variables adapted to filtration $\mathcal{F}_1, \dots, \mathcal{F}_n$ (in particular, $X_s$ may depend on $X_1, \dots, X_{s-1}$). Let $\mathcal{E}_\lambda$ be the event $\mathcal{E}_\lambda = \lrc{\forall s: \E[X_s | \mathcal{F}_{s-1}] \geq \lambda}$. Then, 
 \[
 \prob[\lr{\sum_{s = 1}^n X_s \geq n\lambda} \land \mathcal{E}_\lambda] \leq e^{-n\lambda/8}. 
 \]
 \end{theorem}
We recall that the first $K$ rounds of the algorithm are deterministic, and that each arm is observed at least once. 
We use $O_{[K+1:t], i}$ to refer to the number of observations from rounds $K+1$ to $t$, and we note that $O_{t, i} \geq O_{[K+1:t], i} + 1$. 
We have:
\begin{align*}
    \prob[O_{t, i} \leq \frac{\beta \ln t}{2 \Delta_i^2}]
    & \leq \prob[O_{[K+1:t], i} \leq \frac{\beta \ln t}{2 \Delta_i^2} - 1] \\
    & \leq \prob[O_{[K+1:t], i} \leq \frac{\beta \ln t}{2 \Delta_i^2} \frac{t-K}{t}], 
\end{align*}
where the second step follows from, $\frac{\beta \ln t}{2 \Delta_i^2} - 1 \leq \frac{\beta \ln t}{2 \Delta_i^2} \frac{t-K}{t} \Leftrightarrow \frac{K \beta \ln t}{2t \Delta_i^2} \leq 1$, which is true for $t \geq \tmin(i)$ as
\[
\frac{K \beta \ln t}{2t \Delta_i^2}  \leq \frac{K \beta \ln t}{2 \Delta_i^2} \frac{\Delta_i^4 \lambda \ln K}{4 K^2 \beta^2 \ln^2 t} \leq \frac{\Delta_i^2 \ln K}{8 \ln t} \leq \frac{\Delta_i^2}{8} \leq 1.
\]
We can apply \Cref{th:concentration_ineq} on the $t-K$ random variables $\1[i \in \Nout(I_s)]$ for $s \in \lrs{K+1, t}$ and we get:
\begin{align*}
    \prob[O_{t, i} \leq \frac{\beta \ln t}{2 \Delta_i^2}]
    & \leq \prob[\lr{O_{[K+1:t], i} \leq \frac{\beta \ln t}{2 \Delta_i^2} \frac{t-K}{t} } \land \mathcal{E}_{t, i}] + \prob[\overline{\mathcal{E}_{t, i}}] \\
    & \leq e^{-\frac{t-K}{t}\frac{\beta \ln t}{8 \Delta_i^2}} +  \frac{1}{2} \lr{\frac{\ln t}{t \Delta_i^2}}^{\gamma -2}\\
    & \leq e^{-\frac{3}{4}\frac{\beta \ln t}{8 \Delta_i^2}} +  \frac{1}{2} \lr{\frac{\ln t}{t \Delta_i^2}}^{\gamma -2}\\
    & \leq \lr{\frac{1}{t}}^{\beta/10} +  \frac{1}{2} \lr{\frac{\ln t}{t \Delta_i^2}}^{\gamma -2},
\end{align*}
where we use that $t \geq \tmin(i) \geq 4K$, so $\frac{t - K}{t} \geq \frac{3}{4}$.

\paragraph{A lower bound for $\hat \Delta_{t, i}$}Using Lemma \ref{lem:cb}, we know that the upper and lower confidence bounds satisfy: $\prob[\UCB_{t, i^*} \leq \mu_{i^*} \lor \LCB_{t, i} \geq \mu_{i}] \leq \frac{2}{Kt^{\gamma -1}}$. Then assuming that $\UCB_{t, i^*} \geq \mu_{i^*}$ and $\LCB_{t, i} \leq \mu_{i}$, we have:
\begin{align*}
    \hat \Delta_{t, i} \geq & \LCB_{t, i} - \min_{j \neq i} \UCB_{t, i} \\
     \geq & \LCB_{t, i} - \UCB_{t, i^*} \\
     \geq & \frac{\hat L_{t-1, i}}{O_{t-1, i}} - \sqrt{\frac{\gamma \ln\lr{tK^{1/\gamma}}}{2O_{t-1, i}}}  - \frac{\hat L_{t-1, i^*}}{O_{t-1, i^*}} - \sqrt{\frac{\gamma \ln\lr{tK^{1/\gamma}}}{2O_{t-1, i^*}}}  \\
     = & \frac{\hat L_{t-1, i}}{O_{t-1, i}} +  \sqrt{\frac{\gamma \ln\lr{tK^{1/\gamma}}}{2O_{t-1, i}}} - 2 \sqrt{\frac{\gamma \ln\lr{tK^{1/\gamma}}}{2O_{t-1, i}}} \\
     & - \lr{ \frac{\hat L_{t-1, i^*}}{O_{t-1, i^*}} - \sqrt{\frac{\gamma \ln\lr{tK^{1/\gamma}}}{2O_{t-1, i^*}}}} - 2\sqrt{\frac{\gamma \ln\lr{tK^{1/\gamma}}}{2O_{t-1, i^*}}} \\
     = & \UCB_{t, i} - \LCB_{t, i^*} - 2 \sqrt{\frac{\gamma \ln\lr{tK^{1/\gamma}}}{2O_{t-1, i}}} - 2\sqrt{\frac{\gamma \ln\lr{tK^{1/\gamma}}}{2O_{t-1, i^*}}} \\
     \geq & \Delta_i   - 2 \sqrt{\frac{\gamma \ln\lr{tK^{1/\gamma}}}{2O_{t-1, i}}} - 2\sqrt{\frac{\gamma \ln\lr{tK^{1/\gamma}}}{2O_{t-1, i^*}}}.
\end{align*}
Using the previously derived high probability bounds, assuming that $O_{t, i} \geq \frac{\beta \ln t}{2 \Delta_i^2}$ and $O_{t, i^*} \geq \frac{\beta \ln t}{2 \Delta_i^2}$, and using that $t \geq \tmin(i) \geq K$, we have:
\begin{align*}
    \hat \Delta_{t, i} \geq & \Delta_i   - 2 \sqrt{\frac{\gamma \ln\lr{tK^{1/\gamma}}}{2O_{t-1, i}}} - 2\sqrt{\frac{\gamma \ln\lr{tK^{1/\gamma}}}{2O_{t-1, i^*}}} \\
     \geq & \Delta_i   - 4 \sqrt{\frac{2 \Delta_i^2 \gamma \ln\lr{tK^{1/\gamma}}}{2\beta \ln t}} \\
    \geq & \Delta_i   - 4 \sqrt{\frac{\Delta_i^2 (\gamma + 1) \ln\lr{tK^{1/\gamma}}}{\beta \ln t}} \numberthis{} \label{proof_lb_eq1}\\
    = & \Delta_i  \lr{1 - 4\sqrt{\frac{\gamma + 1}{\beta}}},
\end{align*}
where equation \eqref{proof_lb_eq1} follows from $t \geq K$, so $\gamma \ln\lr{tK^{1/\gamma}} = \gamma \ln\lr{t^\gamma K} \leq \ln\lr{t^\gamma t} = \lr{\gamma + 1} \ln t$. Using that $\beta \geq 64\lr{\gamma + 1}$, we have:

\[ \prob[\hat \Delta_{t, i} \leq \frac{1}{2}\Delta_i] \leq \lr{\frac{\ln t}{t \Delta_i^2}}^{\gamma -2} + \frac{2}{Kt^{\gamma - 1}} + 2\lr{\frac{1}{t}}^{\beta/10} .\]

\section{Proof of Theorem \ref{th:stoc}}
\label{appen:stoch}

In the stochastic regime, we decompose the regret bound into three terms that we bound separately. 
First, during the initial $\tmin = \max_{i: \Delta_i > 0} \lrc{\tmin(i)}$ rounds we use the adversarial bound. Then, in the remaining rounds we bound the contribution of the exponential weights and of the exploration separately.

% We define the conditions under which we prove the following three propositions.

% \begin{definition}
% \label{conditions_prop}
% Let $\lambda \in [1, K]$. 
% We consider Algorithm \ref{alg:exp3G++} played with $\gamma \geq 4$, $\beta \geq 64(\gamma + 1) \geq 320$ exploration rate defined for all $t \geq K + 1$ and $i \in V$ as
% \[
% \varepsilon_{t, i} =  \min \lrc{\frac{1}{2K}, \frac{1}{2} \sqrt{\frac{\lambda \ln K}{tK^2}}, \xi_{t, i}} \qquad \text{ where } \qquad  \xi_{t, i} = \begin{cases}
%     ({\beta \ln t})/({t \hat\Delta_{t, i}^2}) &\qquad \text{if} \qquad i \in S_t,\\
%     {4}/{t^2} &\qquad \text{otherwise.}
% \end{cases}
% \]
% For all $i$ such that $\Delta_i > 0$, we define 
% \[\tmin(i) = \max \left\{t \ge 0: \ \frac{1}{2}\sqrt{\frac{c \ln K}{t K^2}} \leq \frac{ \beta \ln t}{t\Delta_i^2}\right\} \qquad \text{ and }\qquad  \tmin = \max_{i : \Delta_i > 0} \lrc{\tmin(i)}.\]
% \end{definition}

% Let $c \in [1, K]$ and  $S_1, S_2, \dots$ be a sequence of exploration sets generated by playing algorithm \ref{alg:exp3G++} 

\subsection{Control over the Initial Rounds}
% The proof of this Proposition is decomposed into two parts. 
We start by deriving a time independent upper bound on $\tmin(i)$ for all vertices $i$ such that $\Delta_i > 0$. 

\begin{prop}
\label{prop:bound_tmin}
For any constant $c > e^2$, we have:
\[
\max_{t} \lrc{t \leq c (\ln t)^2} \leq 25 c \ (\ln c)^2.
\]
\end{prop}

\begin{proof}
First, we note that for $t = 3$, 
\[c (\ln t)^2 \geq e^2 (\ln 3)^2 \geq 3 = t, \]
so the inequality is fulfilled at $t = 3$. 

Furthermore, $(c (\ln t)^2)' = 2c \frac{\ln t}{t}$ is a decreasing function of $t$ for $t \geq e$ and such that $\lim_{t \to \infty} 2c \frac{\ln t}{t} = 0$, whereas $(t)' = 1$ is constant. Thus, $\max_{t} \lrc{t \leq c (\ln t)^2}$ exists and is solution of \[t = c \lr{\ln t}^2.\]

Let's upper bound this $t$. We denote by $W_{-1}$ the product log function. Then we have:
\begin{align*}
    t = & c \lr{\ln t}^2 \\
    \sqrt{t} = & \sqrt{c} \  \ln t \\
    \sqrt{t} = & 2\sqrt{c}\ln(\sqrt{t}) \\
    x = & b \ln(x) && b = 2\sqrt{c}, \ x = \sqrt{t} \\
    x = & -b W_{-1} \lr{-\frac{1}{b}} && \textnormal{for $b \geq e$}.
\end{align*}
By \citet[Theorem 1]{chatzigeorgiou2013bounds}, we have for $b \geq e$:
\begin{align*}
    -b W_{-1}\lr{-\frac{1}{b}} =& -b W_{-1}\lr{-\exp\lr{-\ln\lr{\frac{b}{e}}- 1}}
    \leq b\left(1 + \sqrt{2\ln\lr{\frac{b}{e}}} + \ln\lr{\frac{b}{e}}\right). 
\end{align*}
Thus, we have that for $c \geq e^2$,
\begin{align*}
    t \leq & 4c \ \lr{1 + \sqrt{2\ln\lr{\frac{2\sqrt{c}}{e}}} + \ln\lr{\frac{2\sqrt{c}}{e}}}^2 \\
    \leq & 4c\ (1 + 2\ln c)^2 \\
    \leq & 25 c \ (\ln c)^2,
\end{align*}
where the last step follows from $(1 + 2\ln c)^2  \leq (\frac{1}{2} \ln (e^2) + 2 (\ln c))^2 \leq (2.5 \ln c)^2 = 6.25 \lr{\ln c}^2$. 
\end{proof}

\begin{prop}
\label{prop:bound_Rtmin}
Under the conditions of Lemma \ref{lem:delta_bounds} with $\gamma = 4$ and $\beta = 320$, the contribution of the initial $\tmin$ rounds to the regret can be bounded as:
\begin{equation}
     R_{\tmin} \leq \frac{80 \beta K}{\Dmin^2} \sqrt{\frac{\tilde \alpha}{\lambda}} \ln\lr{\frac{6 \beta^2 K^2}{\Dmin^4}} + K. 
\end{equation}
\end{prop}

\begin{proof}
By definition, we have $\tmin= \max \left\{t \ge 0 : \  \frac{1}{2}\sqrt{\frac{\lambda \ln K}{t K^2}} \leq \frac{ \beta \ln t}{t\Dmin^2}\right\}$. 
By proposition \ref{prop:bound_tmin}, we have that $\tmin \leq 25 d \lr{\ln d}^2$, where $d = \frac{4 \beta^2 K^2}{\lambda \ln K \Dmin^4}$.
Furthermore, for $d > 1$ we have the bound $d \lr{\ln d}^2 \leq d^2$.
Then, we can use Theorem \ref{th:adv} and deduce that:
\begin{align*}
    R_{\tmin} &\leq 4 \sqrt{\tilde \alpha \ln K \ \tmin} + K\\
    & \leq 4 \sqrt{\tilde \alpha \ln K \ 25 d} \ln (d) + K \\
    & \leq \frac{80 \beta K}{\Dmin^2} \sqrt{\frac{\tilde \alpha}{\lambda}} \ln\lr{\frac{6 \beta^2 K^2}{\Dmin^4}} + K.
\end{align*}
\end{proof}

\subsection{Control over the Exponential Weights}

Proposition \ref{prop:qti} introduced in the proof sketch of Theorem \ref{th:stoc} is based on the following result.

\begin{prop}
% [Proposition \ref{prop:qti} restated]
\label{prop:qti_restated}
Under the conditions of Lemma \ref{lem:delta_bounds} with $\gamma = 4$ and $\beta = 320$, the sum of exponential weights with sequence of learning rates $\eta_1, \eta_2, \dots$ of each suboptimal arm $i$ can be bounded as:
\begin{equation*}
\sum_{t = \tmin(i)}^T  \E[q_{t, i}] \leq \sum_{t = \tmin(i)}^T  \lr{e^{-\frac{1}{2}t\eta_t\Delta_i} + \frac{1}{t} \lr{\frac{\lambda \ln K}{4K^2\beta^2} + \frac{1}{K}}}
\end{equation*}
\end{prop}

To prove this result, we can follow the same derivation as in \citet{SL17}. We want to bound the $q_{t, i}$ for all $i$ such that $\Delta_i > 0$ and $t \geq \tmin(i)$.
First, we note that
\begin{align*}
    q_{t, i}&  =  \frac{\exp(-\eta_t \tilde L_{t, i})}{\sum_{j \in V}\exp(-\eta_t \tilde L_{t, j})} \\
    & =  \frac{\exp(-\eta_t (\tilde L_{t, i} - \tilde L_{t, i^*})) }{\sum_{j \in V}\exp(-\eta_t (\tilde L_{t, j}- \tilde L_{t, i^*})))} \\
    & \leq \exp(-\eta_t (\tilde L_{t, i} - \tilde L_{t, i^*})) \\
    & := \exp(-\eta_t \tilde \Delta_{t, i}), 
\end{align*}
where $i^*$ is the best arm, and where the inequality holds because one term of the sum is $\exp(-\eta_t (\tilde L_{t, i^*} - \tilde L_{t, i^*})) = 1$ and the other terms are positive, so the denominator is greater than $1$. We now want to want to ensure that $ \tilde \Delta_{t, i}:= \tilde L_{t, i} - \tilde L_{t, i^*}$ is close to $t \Delta_i$.  To do so, we want to apply a variant of Bernstein's inequality on the martingale sequence difference $t\Delta_i - \tilde \Delta_{t, i} = \sum_{s = 1}^t X_s$, where each single term of the sequence is defined as $X_s = \Delta_i - (\tilde \ell_{s, i} - \tilde \ell_{s, i^*})$. \\

\begin{theorem}[Bernstein's inequality for martingales]
\label{th:bernsteinMDS}Let $X_1, \dots, X_n$ be a martingale difference sequence with respect to filtration $\mathcal{F}_1,  \dots, \mathcal{F}_n,$ where each $X_j$ is bounded from above, and let $S_i = \sum_{j = 1}^i X_j$ be the associated martingale. Let $\nu_n = \sum_{j =1}^n \E[(X_j)^2| \mathcal{F}_{j-1}]$ and $\kappa_n = \max_{1\leq j \leq n} \lrc{X_j}$.
Then, for any $\delta > 0$:
\[
\prob[\lr{S_n \geq \sqrt{2\nu \ln\lr{\frac{1}{\delta}}} + \frac{\kappa\ln \lr{\frac{1}{\delta}}}{3}} \land \lr{\nu_n \leq \nu} \land \lr{\kappa_n \leq \kappa}] \leq \delta.
\]
\end{theorem}
In order to apply the this theorem, we need to bound $\max_{1\leq s \leq n} \lrc{X_s}$ and $\sum_{s=1}^n \E[(X_s)^2| \mathcal{F}_{s-1}]$.
\paragraph{Control of $\max_{1\leq s \leq t} \lrc{X_s}$} For each $s$ we have:
\begin{align*}
    X_s & =  \Delta_i - (\tilde \ell_{s, i} - \tilde \ell_{s, i^*}) \\
     & \leq 1 + \ell_{s, i^*} \\
     & \leq  1 + \frac{1}{P_{t, i^*}} \\
     & \leq 1 + \max\lrc{2K, 2\sqrt{\frac{sK^2}{\lambda \ln K}}, \frac{s \hat \Delta_{s, i^*}^2}{\beta \ln s}} \numberthis{} \label{proof_cn_1}\\
     & \leq 1.25 \max\lrc{2K, 2\sqrt{\frac{sK^2}{\lambda \ln K}}, \frac{s \hat \Delta_{s, i^*}^2}{\beta \ln s}} \\
\end{align*}
where equation \eqref{proof_cn_1} holds by definition of $o_{t, i^*}$.
% by construction of the exploration set: $i^*$ is observed by at least one arm $j$ in $S_t$ which is played at rate $\varepsilon_{t, j} = \min \lrc{\frac{1}{2K}, \sqrt{\frac{c \ln K}{tK^2}},  \frac{\beta \ln t}{t \hat\Delta_{t, j}^2}}$ and such that $\hat \Delta_{s, i^*} \geq \hat \Delta_{s, j}$.
Using the same argument as in the proof of \Cref{lem:delta_bounds}, we know that $t \geq \tmin(i)$, and if $s \leq \frac{t \Delta_i^2}{ \ln t}$ then $\frac{s \hat \Delta_i^2}{\ln s} \leq \frac{s}{\beta} \leq \frac{ t \Delta_i^2}{ \beta \ln t}$ then:

\begin{align*}
    \prob[ \exists s \leq t: \max\lrc{2K, 2\sqrt{\frac{sK^2}{\lambda \ln K}}, \frac{s \hat \Delta_{s, i^*}^2}{\beta \ln s}} \geq  \frac{ t \Delta_i^2}{ \beta \ln t}] 
    & =  \prob[ \exists s \in \lrs{\frac{t \Delta_i^2}{ \ln t}, t}:   \Delta_{s, i^*} \geq  \Delta_i] \\
    & \leq \prob[ \exists s \in \lrs{\frac{t \Delta_i^2}{ \ln t}, t}:  \Delta_{s, i^*} \geq  \overline{\Delta_{i^*}}],
\end{align*}
because $\Delta_i \geq \Dmin = \bar \Delta_{i^*}$. 
Let $\kappa_t = \max_{1\leq s \leq t} \lrc{X_s}$, and we deduce:
 \[\prob[ \kappa_t \geq  \frac{1.25 t \Delta_i^2}{\beta \ln t}] \leq \prob[ \exists s \in \lrs{\frac{t \Delta_i^2}{ \ln t}, t}:   \Delta_{s, i^*} \geq  \overline{\Delta_{i^*}}].\]

\paragraph{Control of $ \nu_t = \sum_{s = 1}^t \E[(X_s)^2 | \mathcal{F}_{s-1}] $}
We start by looking at each individual element of the sum. 
\begin{align*}
     \E[(X_s)^2 | \mathcal{F}_{s-1}] & = \E[(\Delta_i - (\tilde \ell_{s, i} - \tilde \ell_{s, i^*}))^2 | \mathcal{F}_{s-1}]  \\
     & \leq  \E[(\tilde \ell_{s, i} - \tilde \ell_{s, i^*})^2 | \mathcal{F}_{s-1}]  \\
     & \leq  \E[\tilde \ell_{s, i}^2 | \mathcal{F}_{s-1}] +   \E[\tilde \ell_{s, i^*}^2 | \mathcal{F}_{s-1}],
\end{align*}
where the last equation holds, because for all non-negative $a$ and $b$, we have: $(a - b)^2 \leq  a^2 + b^2$. 

Then, note that:
\begin{align*}
    E[\tilde \ell_{s, i}^2 | \mathcal{F}_{s-1}] \leq \frac{1}{P_{t, i}},
\end{align*}
so 
\begin{align*}
     \E[(X_s)^2 | \mathcal{F}_{s-1}] 
     & \leq  \frac{1}{P_{s, i}} +   \frac{1}{P_{s, i^*}}.
\end{align*}
Using the same argument as before to bound $\frac{1}{P_{s, i^*}}$ and $\frac{1}{P_{s, i}}$, we have:
\begin{align*}
    \prob[ \sum_{s = 1}^t \E\lrs{(X_s)^2 | \mathcal{F}_{s-1}} \geq   \frac{2 t^2 \Delta_i^2}{ \beta \ln t}]
    \leq \ &  \prob[ \exists s \in \lrs{\frac{t \Delta_i^2}{ \ln t}, t}:  \Delta_{s, i^*} \geq  \overline{\Delta_{i^*}}]
    \\ & + \prob[ \exists s \in \lrs{\frac{t \Delta_i^2}{ \ln t}, t}:  \Delta_{s, i^*} \geq  \overline{\Delta_{i^*}}].
    %  \\
    %  \leq \ & \frac{1}{t} \frac{c \ln K}{4K^2\beta^2}.
\end{align*}

Noting that the bounds on $\kappa_t$ and $\nu_t$ depend on the same events, we deduce that:

\begin{align*}
    \prob[ \lr{\kappa_t \geq  \frac{1.25 t \Delta_i^2}{\beta \ln t}} \lor \lr{\nu_t \geq \frac{2 t^2 \Delta_i^2}{ \beta \ln t}} ] 
    \leq \ &  \prob[ \exists s \in \lrs{\frac{t \Delta_i^2}{ \ln t}, t}:   \Delta_{s, i} \geq  \overline{\Delta_{i}}] \\  
    & + \prob[ \exists s \in \lrs{\frac{t \Delta_i^2}{ \ln t}, t}:   \Delta_{s, i^*} \geq  \overline{\Delta_{i^*}}]
    \\
    \leq & \frac{1}{t} \frac{\lambda \ln K}{4K^2\beta^2},
\end{align*} 
 where for the last step we use that for all $j, k \in V$ and $\gamma = 4$, 
\begin{align*}
 \prob[ \exists s \in \lrs{\frac{t \Delta_k^2}{ \ln t}, t}:   \Delta_{s, j} \geq  \overline{\Delta_{j}}] 
     &  \leq \sum_{s = \frac{t \Delta_k^2}{ \ln t}}^t \prob[  \Delta_{s, j} \geq  \overline{\Delta_{j}}] \\
     &  \leq \sum_{s = \frac{t \Delta_k^2}{ \ln t}}^t \frac{1}{s^{3}} \\
      &  \leq \frac{1}{2} \lr{\frac{\ln t}{t\Delta_k^2}}^{2} \\ 
      &\leq \frac{1}{t} \frac{\lambda \ln K}{8K^2\beta^2},
\end{align*}
and the last step follows by definition of $\tmin$.

\paragraph{Control of $\tilde \Delta_{t, i}$} We have that:
\begin{align*}
    \prob[\tilde \Delta_{t, i} \leq \frac{1}{2}t\Delta_i] 
     = \ & \prob[t\Delta_i - \tilde \Delta_{t, i} \geq \frac{1}{2}t\Delta_i] \\
     \leq \ & \prob[\lr{t\Delta_i - \tilde \Delta_{t, i} \geq \frac{1}{2}t\Delta_i}\land \lr{\kappa_t \leq  \frac{1.25 t \Delta_i^2}{\beta \ln t}} \land \lr{\nu_t \leq \frac{4 t^2 \Delta_i^2}{ \beta \ln t}}] \\
     &+ \prob[ \lr{\kappa_t \geq  \frac{1.25 t \Delta_i^2}{\beta \ln t}} \lor \lr{\nu_t \geq \frac{2 t^2 \Delta_i^2}{ \beta \ln t}} ] 
\end{align*}

We set $\nu = \frac{2 t^2 \Delta_i^2}{ \beta \ln t}$, $\kappa = \frac{1.25 t \Delta_i^2}{\beta \ln t}$, $\delta = \frac{1}{Kt}$, and we recall that $\beta = 320$. Then:

\begin{align*}
    \sqrt{2\nu \ln\lr{\frac{1}{\delta}}} + \frac{\kappa \ln \lr{\frac{1}{\delta}}}{3} 
    & = \sqrt{2 \frac{2 t^2 \Delta_i^2}{ \beta \ln t}\ln \lr{Kt}} + \frac{ \frac{1.25 t \Delta_i^2}{\beta \ln t}\ln \lr{Kt}}{3} \\
    & \leq \sqrt{8 \frac{t^2 \Delta_i^2}{ \beta \ln t}\ln \lr{t}} + \frac{ \frac{1.25 t \Delta_i^2}{\beta \ln {t}}\ln t}{3} \numberthis{} \label{proof_bern_KT}\\ 
    & \leq t \Delta_i  \lr{\frac{2\sqrt{2}}{\sqrt{\beta}} + \frac{2.5}{3 \beta}} \\
    &\leq \frac{1}{2}t\Delta_i,
\end{align*}
where equation \eqref{proof_bern_KT} is due to $t \geq \tmin(i) \geq K$, so $\ln \lr{Kt} \leq 2 \ln t$.
%Note: the bound is actually tighter we get 0.16 (using beta = 320, or 0.18 using beta = 256) instead of 1/2. May be able to get tighter constants there
We can then use \Cref{th:bernsteinMDS} and get:
\begin{align*}
    \prob[\tilde \Delta_{t, i} \leq \frac{1}{2}t\Delta_i]  \leq  \frac{1}{t} \frac{\lambda \ln K}{4K^2\beta^2} + \frac{1}{Kt} = \frac{1}{t} \lr{\frac{\lambda \ln K}{4K^2\beta^2} + \frac{1}{K}}.
\end{align*}

Using this bound, summing on $t$ gives
\begin{align*}
\sum_{t = \tmin(i)}^T  \E[q_{t, i}]
& \leq \sum_{t = \tmin(i)}^T  \E \lrs{e^{-\frac{1}{2}t\eta_t\Delta_i} \1[\tilde \Delta_{t, i} \geq \frac{1}{2}t\Delta_i] +  \1[\tilde \Delta_{t, i} \leq \frac{1}{2}t\Delta_i]} \\
& \leq \sum_{t = \tmin(i)}^T  \lr{e^{-\frac{1}{2}t\eta_t\Delta_i} + \frac{1}{t} \lr{\frac{\lambda \ln K}{4K^2\beta^2} + \frac{1}{K}}},
\end{align*}
which finishes the proof.

\subsection{Control over the Exploration}

We now provide a more general version of Proposition \ref{prop:eti}.
\begin{prop}
% [Proposition \ref{prop:eti} restated]
\label{prop:eti_restated}
Let $S_1, S_2, \dots$ be a sequence of exploration sets generated by playing algorithm \ref{alg:exp3G++} under the conditions of Lemma \ref{lem:delta_bounds} with $\gamma = 4$ and $\beta = 320$.
Then, the contribution of the extra exploration can be bounded as:
\begin{equation*}
     \sum_{t = \tmin}^T \sum_{i \,:\, \Delta_i > 0} \Delta_i  \E[\varepsilon_{t, i}] \leq 
    \sum_{t = \tmin}^T \E \lrs{\sum_{i \in S_t: \Delta_i > 0} \frac{4\beta \ln t}{t \Delta_{i}}} + \frac{\lambda  \ln T \ln K}{4K\beta^2} + 12K + 3.
\end{equation*}
\end{prop}

\begin{proof}
By definition of $\xi_{t, i}$, we can decompose the contribution of the extra exploration as follows.
\begin{align*}
    \sum_{t = \tmin}^T  \E[\varepsilon_{t, i}] \sum_{i: \Delta_i > 0} \Delta_i 
    \leq \ &  \sum_{t = \tmin}^T\sum_{i: \Delta_i > 0} \Delta_i    \E[\min\lrc{1, \xi_{t, i}}] \\
    \leq \ & \sum_{t = \tmin}^T \E \lrs{\sum_{i \in S_t: \Delta_i \geq 0}  \Delta_i\E[\min\lrc{1,\frac{\beta \ln t}{t \hat \Delta_{t, i}^2}}]} \\
    & +  \sum_{t = \tmin}^T \sum_{i: \Delta_i > 0} \Delta_i  \frac{4}{t^2},
\end{align*}
where in the first step we use the $\min$ term to ensure that we have an upper bound on this quantity in the cases where the bounds on $\hat \Delta_{t, i}$ do not hold. The last step consists in upper bounding the exploration of arms that are not in $S_t$ by adding $\frac{4}{t^2}$ to all arms, and in counting $\E[\frac{\beta \ln t}{t \hat \Delta_{t, i}^2}]$ only for arms $i$ that are in the exploration set $S_t$.

Thus the second term is bounded as:
\begin{equation}
     \sum_{t = \tmin}^T \sum_{i: \Delta_i > 0} \Delta_i  \frac{4}{t^2} \leq K \sum_{t = 1}^T \frac{4}{t^2} \leq 8K.
\end{equation}
In order to bound the first term, we recall that for $t \geq \tmin$, for any $i \in V$: 
\begin{align*}
    \E[\min\lrc{1,\frac{\beta \ln t}{t \hat \Delta_{t, i}^2}}]
    & \leq \E \lrs{\frac{\beta \ln t}{t \hat \Delta_{t, i}^2} \1[\hat\Delta_{t, i} \geq \frac{1}{2}\overline\Delta_i] + \1[ \hat\Delta_{t, i} \leq \frac{1}{2}\overline\Delta_i ]}\\
    & \leq  \frac{4\beta \ln t}{t \Delta_{i}^2} +  \prob[\hat\Delta_{t, i} \leq \frac{1}{2}\overline\Delta_i] \\
    & \leq \frac{4\beta \ln t}{t \Delta_{i}^2} +  \lr{\frac{\ln t}{t\Delta_i^2}}^{\gamma - 2} + \frac{2}{Kt^{\gamma - 1}} + 2 \lr{\frac{1}{t}}^{\frac{\beta}{10}} \\
    & \leq \frac{4\beta \ln t}{t \Delta_{i}^2} + \frac{1}{t} \frac{\lambda \ln K}{4K^2\beta^2} + \frac{2}{Kt^{3}} + 2 \lr{\frac{1}{t}}^{2},
\end{align*}
which gives:
\begin{align*}
    & \sum_{t = \tmin}^T \E \lrs{\sum_{i \in S_t: \Delta_i > 0}  \Delta_i\E[\min\lrc{1,\frac{\beta \ln t}{t \hat \Delta_{t, i}^2}}]} \\
    & \leq \sum_{t = \tmin}^T \E \lrs{\sum_{i \in S_t: \Delta_i > 0} \Delta_i \lr{\frac{\beta \ln t}{t \Delta_{i}^2} + \frac{1}{t} \frac{\lambda \ln K}{4K^2\beta^2} + \frac{2}{Kt^{3}} + 2 \lr{\frac{1}{t}}^{2}}}\\
    & \leq \sum_{t = \tmin}^T \E \lrs{\sum_{i \in S_t: \Delta_i > 0} \frac{4\beta \ln t}{t \Delta_{i}}} + \frac{\lambda  \ln T \ln K}{4K\beta^2} + 3 + 4K.
\end{align*}
\end{proof}

\subsection{Proof of Theorem \ref{th:stoc}}
\label{appen:proof_th_stoc}
The proof of Theorem \ref{th:stoc} follows from the propositions in this section.

\begin{proof}[Proof of Theorem \ref{th:stoc}]
We want to bound the pseudo-regret of algorithm \ref{alg:exp3G++} run with parameters defined in  Lemma \ref{lem:delta_bounds} with $\gamma = 4$,$\beta = 320$ and $\lambda = \tilde \alpha$. 
The pseudo-regret can be decomposed by treating the first $\tmin$ rounds like in the adversarial case, and by using a refined bound in the stochastic regime.
\begin{align*}
    R_T   & = R_{\tmin} + \sum_{i \,:\, \Delta_i > 0} \sum_{t = \tmin}^T \Delta_i \E[p_{t, i}]\\
    & \leq  R_{\tmin} + \sum_{i \,:\, \Delta_i > 0} \Delta_i \sum_{t = \tmin}^T  \big(\E[q_{t, i}] +   \E[\varepsilon_{t, i}]\big),
\end{align*}
% and the proof follows from applying Propositions \ref{prop:qti}, \ref{prop:eti} and \ref{prop:bound_Rtmin}.
First, we apply Proposition \ref{prop:bound_Rtmin} with $c = \tilde \alpha$, and deduce that:
\begin{equation}
     R_{\tmin} \leq  \frac{80 \beta K}{\Dmin^2} \ln\lr{\frac{6 \beta^2 K^2}{\Dmin^4}} + K.\label{proof_stoch:eq1}
\end{equation}

Then we bound the contribution of exponential weights by applying Proposition \ref{prop:qti_restated} with $\lambda = \tilde \alpha$ and $\eta_t = \sqrt{\frac{\ln K}{2 \tilde \alpha t}}$, which gives:
\begin{align*}
\sum_{t = \tmin(i)}^T  \E[q_{t, i}] & \leq \sum_{t = \tmin(i)}^T  \lr{e^{-\frac{1}{2}t\eta_t\Delta_i} + \frac{1}{t} \lr{\frac{\tilde \alpha \ln K}{4K^2\beta^2} + \frac{1}{K}}} \\
 & \leq \sum_{t = \tmin(i)}^T  \lr{e^{-\Delta_i \sqrt{\frac{\ln K}{2 \tilde \alpha}}\sqrt{t}} + \frac{1}{t} \lr{\frac{\tilde \alpha \ln K}{4K^2\beta^2} + \frac{1}{K}}} \\
 & \leq \frac{6 \tilde \alpha}{\Delta_i^2} + \ln T \lr{\frac{\tilde \alpha \ln K}{4K^2\beta^2} +\frac{1}{K}},
\end{align*}
and then:
\begin{align*}
    \sum_{i: \Delta_i > 0} \Delta_i \sum_{t = \tmin}^T  \E[q_{t, i}] 
    & \leq \sum_{i: \Delta_i > 0} \Delta_i \lr{\frac{6 \tilde \alpha}{\Delta_i^2} + \ln T \lr{\frac{\tilde \alpha \ln K}{4K^2\beta^2} + \frac{1}{K}}} \\
    & \leq \ln T \lr{\frac{\tilde \alpha \ln K}{4K\beta^2} + 1} + \sum_{i: \Delta_i \geq 0} \frac{6 \tilde \alpha}{\Delta_i}, \numberthis{} \label{proof_stoch:eq2}
\end{align*}
where the last step follows from Lemma \ref{imp:TS18_lem_3}.
Furthermore, we bound the contribution of the extra exploration by applying Proposition \ref{prop:eti_restated} with $\lambda = \tilde \alpha$, which gives:
\begin{align*}
     \sum_{i \,:\, \Delta_i > 0} \Delta_i \sum_{t = \tmin(i)}^T  \E[\varepsilon_{t, i}] &\leq 
    \sum_{t = 1}^T \E \lrs{\sum_{i \in S_t: \Delta_i > 0} \frac{4\beta \ln t}{t \Delta_{i}}} + \frac{\tilde \alpha  \ln T \ln K}{4K\beta^2} + 12K + 3 \\ 
     & \leq \max_{Ind \in \mathcal{I}(G)} \lrc{ \sum_{i \in Ind: \Delta_i > 0} \frac{4\beta \ln^2 T}{\Delta_{i}}} + \frac{\tilde \alpha  \ln T \ln K}{4K\beta^2} + 12K  + 3, \numberthis{} \label{proof_stoch:eq3}
\end{align*}
where the last step follows from Proposition \ref{prop:exploration_set}: by definition, for all $t$, $S_t$ is a strongly independent set on $G$, and we can upper bound by taking the maximum over all the strongly independent sets of $G$.

Finally, summing over equations \eqref{proof_stoch:eq1}, \eqref{proof_stoch:eq2} and \eqref{proof_stoch:eq3} finishes the proof.
\end{proof}

\section{Extension to Graphs that Change over Time}
\label{appen:sec_var}

We recall that for this version of the problem, we use $\varepsilon_t$ given by:
% \begin{equation*}
%     \varepsilon_{t, i} = \min \lrc{\frac{1}{2K}, \frac{1}{2} \sqrt{\frac{\ln K}{t K^2}}, \xi_{t, i}}, 
% \end{equation*}
\[
\varepsilon_{t, i} =  \min \lrc{\frac{1}{2K}, \frac{1}{2} \sqrt{\frac{\ln K}{tK^2}}, \xi_{t, i}} \qquad \text{ where } \qquad  \xi_{t, i} = \begin{cases}
    ({\beta \ln t})/({t \hat\Delta_{t, i}^2}) &\qquad \text{if} \qquad i \in S_t,\\
    {4}/{t^2} &\qquad \text{otherwise.}
\end{cases}
\]
and that the learning rate is defined from index $t \geq K + 1$ by:
\begin{align*}
    \eta_t = \sqrt{\frac{\ln K}{2\sum_{s = K}^{t-1}\theta_s}}, \quad \text{ where } \theta_t = \sum_{i \in V} \frac{p_{t, i}}{P_{t, i}}.
\end{align*}
As the quantities $p_{t, i}$ are not defined for $t \geq K$, we set $\theta_K := K$. 
This ensures that the learning rate is well defined and non-increasing at all the rounds where we use exponential weights.

To take advantage of the structure of the feedback graph and depend on the independence number of the graphs rather than their strong independence number, we rely on Lemma 5 from\citet{ACBDK15}.

\begin{lemma}[Lemma 5 \citep{ACBDK15}] \label{imp:ACBDK15_lem5}
Let $G = (V, E)$ be a directed graph with $|V| = K$, in which each node $i \in V$ is assigned a positive weight $w_i$. Assume that $\sum_{i \in V} w_i \leq 1$, and that $w_i \geq \epsilon$ for all $i \in V$ for some constant $0 < \epsilon < \frac{1}{2}$. Then
\[ \sum_{i \in V} \frac{w_i}{w_i + \sum_{j \in \Nin(i)} w_j} \leq 4 \alpha \ln \lr{\frac{4K}{\alpha \epsilon}}, 
\]
where $\alpha = \alpha(G)$ is the independence number of $G$.
\end{lemma}

\begin{proof}[Proof of Theorem \ref{th:adv_time_var}]
\textbf{Adversarial Regime}
% \label{appen:adv_var}
In the adversarial regime, the proof follows the analysis with a fixed feedback graph.

Using the same arguments as in the proof of Theorem \ref{th:adv}, equations \ref{proof_adv:reg_decomp} and \ref{proof_adv:qtiPti} hold and we have:

\begin{align*}
    \mathcal{R}_T 
    & \leq K + \E[ \sum_{t = K+1}^T \sum_{i = 1}^K q_{t, i} \E_t\lrs{\tilde \ell_{t, i}} -  \sum_{t = K+1}^T \E_t\lrs{\tilde \ell_{t, i^*}}] + \E[ \sum_{t = K+1}^T \sum_{i = 1}^K \varepsilon_{t, i} \E_t\lrs{\tilde \ell_{t, i}} ] \\
     & \leq K + \E[ \E_t \lrs{\frac{\ln K}{\eta_T}} ]+ \E[\sum_{t = K+ 1}^{T}  \E_t \lrs{\sum_{i \in V} \frac{\eta_t}{2} \frac{q_{t, i}}{P_{t, i}}}] + \E[ \sum_{t = K+1}^T \sum_{i = 1}^K  \E_t\lrs{ \varepsilon_{t, i} \tilde \ell_{t, i}} ]. \numberthis{} \label{proof:adv_time_var_reg}
\end{align*}

Note that we can apply Lemma \ref{imp:ACBDK15_lem5} where each round $t$ derives a bound that depends on $\alpha_t$. We deduce that:
\[ \theta_t =  \sum_{i \in V} \frac{p_{t, i}}{P_{t, i}} \leq 8 \alpha_t \ln \lr{KT}.\] 

Using the new definition of $\eta_t$, the first term becomes:
\begin{align*}
   \E[\E_t \lrs{\frac{\ln K}{\eta_T}}]  \leq   \sqrt{2 \ln K} \  \E[\E_t \lrs{ \sqrt{\sum_{s = K}^{T-1}\theta_s}}] \leq 4 \sqrt{\ln K} \ \sqrt{\ln \lr{KT} \sum_{s = K}^{T-1} \alpha_t}.
\end{align*}

The second term can be bounded as:
\begin{align*} 
    \E[\sum_{t = K+ 1}^{T}  \E_t \lrs{\sum_{i \in V} \frac{\eta_t}{2} \frac{q_{t, i}}{P_{t, i}}}] & \leq  \E[\sum_{t = K+ 1}^{T}  \E_t \lrs{\eta_t \sum_{i \in V} \frac{p_{t, i}}{P_{t, i}}}] \\
    & \leq \sqrt{\ln K} \ \E[\sum_{t = K+ 1}^{T}  \E_t \lrs{\frac{\theta_t}{\sqrt{2 \sum_{s = K}^{t-1}\theta_s}}} ] \\
    & \leq   \sqrt{2\ln K} \  \E[ \E_t\lrs{\sqrt{\sum_{t = K+ 1}^{T}   \theta_t}}] + \sqrt{K} \numberthis{} \label{proof:adv_time_var_st1} \\
     & \leq 4 \sqrt{\ln K}   \sqrt{\ln \lr{KT} \sum_{t = K+ 1}^{T} \alpha_t } + \sqrt{K},
\end{align*}
where equation \ref{proof:adv_time_var_st1} follows from Lemma \ref{lem:sum_offset}.

For the last term, we recall that we bounded $\varepsilon_{t, i} \leq \frac{1}{2} \sqrt{\frac{\ln K}{tK^2}}$, and we have:

\begin{align*}
    \E[ \sum_{t = K+1}^T \sum_{i = 1}^K \varepsilon_{t, i} \E_t\lrs{\tilde \ell_{t, i}} ] 
    & \leq \sum_{t = K+1}^T \E[  \sum_{i = 1}^K  \frac{1}{2} \sqrt{\frac{\ln K}{tK^2}} \E_t\lrs{\tilde \ell_{t, i}} ]  \\
    & \leq \sum_{t = 1}^T \frac{1}{2} \sqrt{\frac{\ln K}{t}} \\
    & \leq \sqrt{T \ln K}.
\end{align*}

and using those three bounds in equation \ref{proof:adv_time_var_reg} finishes the proof.
\begin{align*}
    \mathcal{R}_T & \leq K + 4 \sqrt{\ln K} \ \sqrt{\ln \lr{KT} \sum_{s = K}^{T-1} \alpha_t}
+ 4 \sqrt{\ln K}   \sqrt{\ln \lr{KT} \sum_{t = K+ 1}^{T} \alpha_t } + \sqrt{K} + \sqrt{T \ln K} \\
& \leq  9 \sqrt{\ln K} \sqrt{\ln \lr{KT}} \sqrt{\sum_{t = 1}^T \alpha_t}  + 2K. 
\end{align*}

\textbf{Stochastic Regime}
We recall that for this problem we can lower bound the probability of observing arm $i$ at round $t$ by:
\begin{equation}
    o_{t, i} = \min \lrc{\frac{1}{2K}, \frac{1}{2} \sqrt{\frac {\ln K}{t K^2}} , \frac{\beta \ln t}{t  \hat \Delta_{t, i}^2}}, \label{eq_lb_var_oti}
\end{equation}
without requiring knowledge of the independence number of any of the graphs $G_t$. This rate corresponds to choosing $\lambda = 1$ in the definition of $\varepsilon$ for fixed graphs and means that we are playing each arm at a lower rate than if we knew the structure of the graph. Choosing $\lambda = 1$ affects $\tmin$, which is now given by
\begin{align*}
    \tmin & = \max \left\{t \ge 0 : \ t \leq \frac{ 4 \beta^2 K^2}{\Dmin^4 \ln K} \lr{\ln t}^2\right\} \\
    & \leq 25 \frac{ 4 \beta^2 K^2}{\Dmin^4 \ln K} \lr{\ln \lr{ \frac{ 4 \beta^2 K^2}{\Dmin^4 \ln K}}}^2.
\end{align*}
% which is larger than the value of $\tmin$ we had for fixed graphs, as we are not taking advantage of the independence numbers and limiting exploration.
Combining this definition of $\tmin(i)$ with the lower bound on $o_{t, i}$ allows us to use Lemma \ref{lem:delta_bounds}.

We recall that the pseudo-regret can be decomposed as follows (equation \eqref{decomp_stoch_reg2}) 
\begin{equation*}
    R_T \leq  R_{\tmin} + \sum_{i \,:\, \Delta_i > 0} \Delta_i \sum_{t = \tmin(i)}^T  \big(\E[q_{t, i}] +   \E[\varepsilon_{t, i}]\big).
\end{equation*}
so we bound each term individually. The first term is bounded by applying the first part of this theorem, using that for all $t$, $\alpha_t \leq K$,
and $\tmin \leq 25 d \lr{\ln d}^2$  
where  $d =  \frac{ 4 \beta^2 K^2}{\Dmin^4}$. We also recall that for $d \geq 1$, $ \lr{\ln d}^2 \leq d^2$.
\begin{align*}
    R_{\tmin} &\leq 9 \sqrt{K \tmin} \ln(K\tmin) + 2K\\
    & \leq 9 \sqrt{K \ 25 d} \ln (d) \ln\lr{25Kd^2} + 2K \\
    & \leq \frac{180 \beta K^{3/2}}{\Dmin^2} \lr{\ln\lr{\frac{20\beta^2 K^{5/2}}{\Dmin^4}}}^2 + 2K. \numberthis{} \label{proof_rtmin_var}
\end{align*}

The second term follows from using Proposition \ref{prop:qti_restated} with $\lambda = 1$. We get that for all $i$ such that $\Delta_i > 0$:

\begin{equation*}
\sum_{t = \tmin(i)}^T  \E[q_{t, i}] \leq \sum_{t = \tmin(i)}^T  \lr{e^{-\frac{1}{2}t\eta_t\Delta_i} + \frac{1}{t} \lr{\frac{\ln K}{4K^2\beta^2} + \frac{1}{K}}}.
\end{equation*}

We can then use that for all $t$, $\theta_t \in [1, K]$, which implies that $\eta_t \geq \sqrt{\frac{\ln K}{2tK}}$ and deduce that
\begin{align*}
    \sum_{t = \tmin(i)}^T  e^{-\frac{1}{2}t\eta_t\Delta_i}
    & \leq  \sum_{t = 1}^T  e^{-\frac{1}{2}t\sqrt{\frac{\ln K}{2tK}}\Delta_i} \\
    & = \sum_{t = 1}^T  e^{-\sqrt {t} \sqrt{\frac{\ln K}{8K}}\Delta_i} \\
    & \leq \frac{16 K}{\Delta_i^2 \ln K},
\end{align*}
where the last step follows from Lemma \ref{imp:TS18_lem_3}.
We use $\sum_{t = 1}^T \frac{1}{t} \leq \ln T$ to bound the second part of the sum, and we get:
\begin{equation}
    \sum_{i:\Delta_i > 0} \Delta_i \sum_{t = \tmin(i)} \E[q_{t, i}] \leq \ln T \lr{\frac{\ln K}{4K\beta^2} + 1}  + \sum_{i: \Delta_i > 0} \frac{16 K}{\Delta_i}. \label{proof_qti_var}
\end{equation}

The last term follows Proposition \ref{prop:eti_restated} with $\lambda = 1$, which gives:

\begin{align*}
    \sum_{t = \tmin}^T \E \lrs{\sum_{i \in S_t: \Delta_i \geq 0}  \Delta_i\E[\varepsilon_{t, i}] }
    & \leq \sum_{t = 1}^T \E \lrs{\sum_{i \in S_t: \Delta_i \geq 0} \frac{4\beta \ln t}{t \Delta_{i}}} + \frac{\ln T \ln K}{4K\beta^2} + 12K + 3.
\end{align*}    
    % \\
%      & \leq \max_{S \subset V: |S| = \tilde \alpha
%      } \lrc{ \sum_{i \in S: \Delta_i > 0} \frac{4\beta \ln^2 T}{\Delta_{i}}} + \frac{  \ln T \ln K}{4K\beta^2} + 12K + 3, \numberthis{}  \label{proof_eti_var}
% \end{align*}
% where $\tilde \alpha$ is an upper bound in the independence number of the bi-directional subgraph of all $G_t$.

 We recall that because $\sum_{i \in S_t: \Delta_i \geq 0}  \Delta_i\E[\varepsilon_{t, i}] \leq 1$ for all $t$, we can skip rounds that have the largest upper bound on $S_t$ by upper bounding the contribution of such rounds by $1$. 
Let $\tilde \alpha_n$ be the $n^{th}$ largest element in the set containing the strong independence number of  $G_t$, with $t \in [1, T]$.
Then we can upper bound the first term in Proposition \ref{prop:eti_restated} as:
\begin{align*}
    \sum_{t = 1}^T \E \lrs{\sum_{i \in S_t: \Delta_i \geq 0} \frac{4\beta \ln t}{t \Delta_{i}}}
     & \leq \inf_{0 \leq n \leq T} \lrc{\max_{S \subset V: |S| = \tilde \alpha_n
     } \lrc{ \sum_{i \in S: \Delta_i > 0} \frac{4\beta \ln^2 T}{\Delta_{i}}} + n}. 
    %  \numberthis{}  \label{proof_eti_var_cor}
\end{align*}
This gives
\begin{align*}
    \sum_{t = \tmin}^T \E \lrs{\sum_{i \in S_t: \Delta_i \geq 0}  \Delta_i\E[\varepsilon_{t, i}] }
    \leq & \inf_{0 \leq n \leq T} \lrc{\max_{S \subset V: |S| = \tilde \alpha_n
     } \lrc{ \sum_{i \in S: \Delta_i > 0} \frac{4\beta \ln^2 T}{\Delta_{i}}} + n} \\
     &+ \frac{  \ln T \ln K}{4K\beta^2} + 12K + 3. \numberthis{}  \label{proof_eti_var}
\end{align*} 

We finish the proof by summing on equations \eqref{proof_rtmin_var}, \eqref{proof_qti_var} and \eqref{proof_eti_var}.
\begin{align*}
     R_T \leq & \frac{180 \beta K^{3/2}}{\Dmin^2} \lr{\ln\lr{\frac{20\beta^2 K^{5/2}}{\Dmin^4}}}^2 + K + \max_{S \subset V: |S| = \tilde \alpha
     } \lrc{ \sum_{i \in S: \Delta_i > 0} \frac{4\beta \ln^2 T}{\Delta_{i}}} \\
    & +  \ln T \lr{\frac{\ln K}{2K\beta^2} + 1} + \sum_{i: \Delta_i > 0} \frac{16 K}{\Delta_i} + 12K + 3 \\
     \leq &  \inf_{0 \leq n \leq T} \lrc{\max_{S \subset V: |S| = \tilde \alpha_n
     } \lrc{ \sum_{i \in S: \Delta_i > 0} \frac{4\beta \ln^2 T}{\Delta_{i}}} + n} \\
     &+ 2 \ln T  +  \sum_{i: \Delta_i > 0} \frac{16 K}{\Delta_i} + \frac{181 \beta K^{3/2}}{\Dmin^2} \lr{\ln\lr{\frac{20\beta^2 K^{5/2}}{\Dmin^4}}}^2.
\end{align*}
\end{proof}

\end{document}